\def\isarxiv{1}

\def\paperTitle{Towards High-Order Mean Flow Generative Models: Feasibility, Expressivity, and Provably Efficient Criteria}

\def\paperAuthor{
Yang Cao\thanks{\texttt{ycao4@wyomingseminary.org}. Wyoming Seminary.}
\and
Yubin Chen\thanks{\texttt{abinzzz1227@gmail.com}. San Jose State University.}
\and
Zhao Song\thanks{\texttt{magic.linuxkde@gmail.com}. University of California, Berkeley.}
\and
Jiahao Zhang\thanks{\texttt{ml.jiahaozhang02@gmail.com}.}
}

\ifdefined\isarxiv
\ifdefined\paperID
\PackageError{CaoChongChengXiang}{You must remove paperID for arXiv submission!}
\fi
\documentclass[11pt]{article}
\usepackage[numbers]{natbib}
\else
\documentclass[10pt,twocolumn,letterpaper]{article}

\usepackage[review,algorithms]{wacv}      
\fi

\ifdefined\isarxiv

\usepackage{amsmath}
\usepackage{amsthm}
\usepackage{amssymb}
\usepackage{algorithm}
\usepackage{subfig}
\usepackage{algpseudocode}
\usepackage{graphicx}
\usepackage{grffile}
\usepackage{wrapfig,epsfig}
\usepackage{url}
\usepackage{xcolor}
\usepackage{epstopdf}

\usepackage{bbm}
\usepackage{dsfont}

\else 
\usepackage{amsmath}
\usepackage{amsthm}
\usepackage{amssymb}
\usepackage{algorithm}
\usepackage{algpseudocode}
\usepackage{graphicx}
\usepackage{grffile}
\usepackage{wrapfig,epsfig}
\usepackage{url}
\usepackage{xcolor}
\usepackage{epstopdf}

\usepackage{bbm}
\usepackage{dsfont}

\definecolor{wacvblue}{rgb}{0.21,0.49,0.74}
\usepackage[pagebackref,breaklinks,colorlinks,allcolors=wacvblue]{hyperref}

\def\wacvPaperID{\paperID} 
\def\confName{WACV}
\def\confYear{2026}

\fi
 
\allowdisplaybreaks

\ifdefined\isarxiv

\usepackage{tikz}
\usepackage{hyperref}  
\hypersetup{colorlinks=true,citecolor=blue,linkcolor=blue} 
\usetikzlibrary{arrows}
\usepackage[margin=1in]{geometry}

\else

\fi
 
\graphicspath{{./figs/}}

\theoremstyle{plain}
\newtheorem{theorem}{Theorem}[section]
\newtheorem{lemma}[theorem]{Lemma}
\newtheorem{definition}[theorem]{Definition}

\newtheorem{fact}[theorem]{Fact}
\newtheorem{remark}[theorem]{Remark}

\newcommand{\ov}{\overline}
\newcommand{\N}{\mathcal{N}}
\newcommand{\R}{\mathbb{R}}

\renewcommand{\d}{\mathrm{d}}

\DeclareMathOperator*{\E}{{\mathbb{E}}}
\DeclareMathOperator*{\var}{\mathrm{Var}}
\DeclareMathOperator*{\Z}{\mathbb{Z}}

\DeclareMathOperator{\poly}{poly}

\DeclareMathOperator{\diag}{diag}

\newcommand{\pprior}{{\cal N}(\mu,\sigma^2 I_d)}
\newcommand{\pdata}{p_{\mathrm{data}}}

\newcommand{\MeanFlow}{\mathsf{MeanFlow}}

\newcommand{\IVP}{\mathsf{IVP}}

\begin{document}

\ifdefined\isarxiv

\date{}
\title{\paperTitle}
\author{\paperAuthor}

\else

\title{\paperTitle}

\author{First Author\\
Institution1\\
Institution1 address\\
{\tt\small firstauthor@i1.org}
\and
Second Author\\
Institution2\\
First line of institution2 address\\
{\tt\small secondauthor@i2.org}
}
\maketitle
\fi

\ifdefined\isarxiv
\begin{titlepage}
  \maketitle
  \begin{abstract}
    Generative modelling has seen significant advances through simulation-free paradigms such as Flow Matching, and in particular, the $\MeanFlow$ framework, which replaces instantaneous velocity fields with average velocities to enable efficient single-step sampling. In this work, we introduce a theoretical study on Second-Order $\MeanFlow$, a novel extension that incorporates average acceleration fields into the $\MeanFlow$ objective. We first establish the feasibility of our approach by proving that the average acceleration satisfies a generalized consistency condition analogous to first-order $\MeanFlow$, thereby supporting stable, one-step sampling and tractable loss functions. We then characterize its expressivity via circuit complexity analysis, showing that under mild assumptions, the Second-Order $\MeanFlow$ sampling process can be implemented by uniform threshold circuits within the $\mathsf{TC}^0$ class. Finally, we derive provably efficient criteria for scalable implementation by leveraging fast approximate attention computations: we prove that attention operations within the Second-Order $\MeanFlow$ architecture can be approximated to within $1/\poly(n)$ error in time $n^{2+o(1)}$. Together, these results lay the theoretical foundation for high-order flow matching models that combine rich dynamics with practical sampling efficiency.

  \end{abstract}
  \thispagestyle{empty}
\end{titlepage}

\newpage

\else

\begin{abstract}
Generative modelling has seen significant advances through simulation-free paradigms such as Flow Matching, and in particular, the $\MeanFlow$ framework, which replaces instantaneous velocity fields with average velocities to enable efficient single-step sampling. In this work, we introduce a theoretical study on Second-Order $\MeanFlow$, a novel extension that incorporates average acceleration fields into the $\MeanFlow$ objective. We first establish the feasibility of our approach by proving that the average acceleration satisfies a generalized consistency condition analogous to first-order $\MeanFlow$, thereby supporting stable, one-step sampling and tractable loss functions. We then characterize its expressivity via circuit complexity analysis, showing that under mild assumptions, the Second-Order $\MeanFlow$ sampling process can be implemented by uniform threshold circuits within the $\mathsf{TC}^0$ class. Finally, we derive provably efficient criteria for scalable implementation by leveraging fast approximate attention computations: we prove that attention operations within the Second-Order $\MeanFlow$ architecture can be approximated to within $1/\poly(n)$ error in time $n^{2+o(1)}$. Together, these results lay the theoretical foundation for high-order flow matching models that combine rich dynamics with practical sampling efficiency.
\end{abstract}

\fi



\section{Introduction}

Generative modeling has witnessed remarkable progress in recent years, driven by the development of flexible, simulation-free paradigms such as Flow Matching (FM)~\cite{lcb+23,av23,lgl23}. By regressing vector fields between latent noise distributions and complex data distributions, flow matching provides an efficient alternative to continuous normalizing flows~\cite{crbd18}, achieving state-of-the-art performance in image and video generation~\cite{fhla24,aaa+24,ekb+24,jsl+25}.

Recently, $\MeanFlow$~\cite{gdb+25} has emerged as a promising variant of flow matching. Unlike traditional FM models that predict instantaneous velocity fields, $\MeanFlow$ learns average velocities across time intervals and leverages Jacobian-vector product (JVP) computations for training. This not only simplifies optimization but also enables fast inference via a consistency condition that facilitates efficient single-step sampling~\cite{sdcs23,sd24}. Given its practical advantages in training stability and inference speed, $\MeanFlow$ has become an important direction in modern generative modeling.

In this paper, we explore a fundamental theoretical question:
\begin{quote}
    \textit{Can the $\MeanFlow$ framework be extended beyond first-order dynamics to incorporate second-order flow information such as acceleration fields?}
\end{quote}
The motivation for this question is rooted in recent efforts in high-order flow matching~\cite{gll+25,cgl+25,cgl+25_nrflow,lss+25}, which demonstrate that modeling both velocity and acceleration fields enhances expressivity and improves generation quality. Inspired by this, we propose and study Second-Order $\MeanFlow$, a novel framework that generalizes average velocity to average acceleration, thereby creating a more expressive and consistent generative flow.

We provide a comprehensive theoretical analysis of Second-Order $\MeanFlow$, addressing its feasibility, expressivity, and computational efficiency. Specifically, our contributions can be summarized as follows:
\begin{itemize}
    \item {\bf Feasibility (Theorems~\ref{thm:consistency_mean_acceleration},~\ref{thm:equi_2nd_meanflow_loss}, and~\ref{thm:efficiency_2nd_meanflow_loss}).} We introduce a formulation of average acceleration and show it satisfies a generalized consistency condition, enabling stable and fast sampling analogous to first-order $\MeanFlow$.
    \item {\bf Expressivity (Theorem~\ref{thm:2nd_thm:meanflow_sampling_tc0}).} We analyze the computational expressivity of Second-Order $\MeanFlow$ through the lens of circuit complexity theory, showing that the model can be simulated within the $\mathsf{TC}^0$ class under reasonable assumptions.
    \item {\bf Provably Efficient Criteria (Theorem~\ref{thm:existence_quadratic_alg}).} We prove that Second-Order $\MeanFlow$ can use fast approximation attention computation, to achieve $n^{2+o(1)}$ computation time complexity with less than $1/\poly(n)$ approximation error.
\end{itemize}

Together, our findings establish Second-Order $\MeanFlow$ as a theoretically grounded and practically feasible generalization of the $\MeanFlow$ family. We believe this work opens up new directions in the development of high-order flow models and their applications to fast and expressive generative modeling.

{\bf Roadmap.}
In Section~\ref{sec:preli}, we provide the preliminary for this work.
In Section~\ref{sec:feasibility_meanflow}, we formally prove the feasibility of Second-Order $\MeanFlow$.
In Section~\ref{sec:circuit_complexity}, we show the circuit complexity of $\MeanFlow$ and Second-Order $\MeanFlow$.
In Section~\ref{sec:provably_efficiency}, we provide the provably efficient analysis.
In Section~\ref{sec:discussion}, we discuss the motivation and practical implications of higher-order $\MeanFlow$.
Finally, we conclude our work in Section~\ref{sec:conclusion}.

\section{Preliminary} \label{sec:preli}

In this section, we first introduce the notation used across this work. Then, we present the preliminary in flow matching. Next, we show the previous work of second-order flow matching.

\subsection{Notation}\label{sec:preliminary:notation}
We use $\Pr[]$ to denote the probability. We use $\E[]$ to denote the expectation. We use $\var[]$ to represent the variance.
Let $\|x\|_p$ be the $\ell_p$-norm of a vector $x \in \R^n$, i.e. $\|x\|_p := (\sum_{i=1}^n |x_i|^p)^{1/p}$. Specifically, we define $\ell_1$, $\ell_2$, and $\ell_\infty$ vector norm as 
$\|x\|_1 := \sum_{i=1}^n |x_i|$, $\|x\|_2 := (\sum_{i=1}^n x_i^2)^{1/2}$, and $\|x\|_{\infty} := \max_{i \in [n]} |x_i|$. 
We use $p(\cdot)$ to denote the probability density function for a distribution. Let ${\sf Uniform}[a,b]$ be a uniform distribution in the interval $[a,b]$. Let $\circ$ denote the Hadamard product. We employ the operator $||$ to denote merging two matrices or vectors along the final dimension.

\subsection{Basic Concepts in Flow Matching}\label{sec:preliminary:flow_matching}

Generative models like Flow Matching~\cite{av23,lcb+23,lgl23} are designed to learn the velocity fields that transform one probability distribution into another. Specifically, it learns to align trajectories between a prior noise distribution that we can easily sample from, to a data distribution (e.g., the distribution of real-world images or videos). We begin by presenting the definition of trajectories.

\begin{definition}[Trajectory]\label{def:trajectory}

Let $t \in [0,1]$ represent the timestep. Let $\alpha_t,\beta_t:[0,1] \to \R$ denote interpolation functions, where $\alpha_0 = \beta_1=1$ and $\alpha_1 = \beta_0 = 0$. The prior distribution is $d$-dimensional multivariate Gaussian $\epsilon \sim {\cal N}(\mu,\sigma^2 I_d)$ and the data distribution is $x \sim \pdata(x)$. Then, we define the trajectory $z_t$ as follows:
    $
        z_t := \alpha_t x + \beta_t \epsilon. 
    $
\end{definition}

Considering both endpoints of the trajectory, we can simply obtain the following fact.

\begin{fact}
    Let $z_t$ be defined as in Definition~\ref{def:trajectory}. Then, we have $z_0 \sim \pdata$ and $z_1 \sim \pprior$.
\end{fact}

Next, we define two velocity fields between the noise distribution and the data distribution. We first define the conditional velocity for one specific sample. 

\begin{definition}[Conditional velocity, implicit in page 4 on~\cite{lcb+23}]\label{def:conditional_velocity}
Let $t \in [0,1]$ denote the timestep. The trajectory $z_t \in \R^d$ is defined in Definition~\ref{def:trajectory}. Then, we define the conditional velocity $v_t$ as follows: 
    $
         v_t := \frac{\d z_t}{\d t}.
    $
\end{definition}

Computing the derivative of $z_t$ w.r.t. $t$, we have the following fact.

\begin{fact}
For any $x\sim\pdata$ and $\epsilon\sim\pprior$, we have
$
v_t = \frac{\d \alpha_t}{\d t} x + \frac{\d \beta_t}{\d t} \epsilon.
$
\end{fact}

Then, we provide two basic facts on the probability along the trajectory. 
\begin{fact}[Conditional probability path, implicit in page 4 of~\cite{lcb+23}]
$
    p_t(z_t | x) = \N(z_t | x \cdot \alpha_t, \beta_t^2 I_d)
$.
\end{fact}

\begin{fact}[Marginal probability path, implicit in page 3 of~\cite{lcb+23}]
$
        p_t(z_t) = \sum_{x} p_t(z_t | x) \pdata(x) 
$.
\end{fact}

Now, we define the marginal velocity, which is the expected velocity field for all possible samples. 
\begin{definition}[Marginal velocity, implicit in page 3 on~\cite{lcb+23}]\label{def:marginal_velocity}
Let $t \in [0,1]$ denote the timestep. The trajectory $z_t \in \R^d$ is defined in Definition~\ref{def:trajectory}, and the  conditional velocity $v_t \in \R^d$ defined in Definition~\ref{def:conditional_velocity}. Then, we define the marginal velocity as follows:
    $
        v(z_t,t) := \E_{v_t \sim p_t(v_t|z_t)} [v_t].
    $
\end{definition}

\subsection{Second-Order Flow Matching}

Recently, a wide range of works~\cite{lss+25,cgl+25,cgl+25_nrflow} have shown that flow matching models can gain advantages from parameterizing both velocity fields and acceleration fields with respect to the trajectory. 

{\bf Basic Concepts.}
These models are built on the following formulations of acceleration fields. Similar to first-order flow matching, we first define the conditional acceleration.

\begin{definition}[Conditional acceleration]\label{def:conditional_acceleration}
Let $t \in [0,1]$ denote the timestep. The trajectory $z_t \in \R^d$ is defined in Definition~\ref{def:trajectory}. Then, we define the conditional acceleration $a_t$ as follows:
    $
         a_t := \frac{\d^2 z_t}{\d t^2}.
    $
\end{definition}

Computing the second-order derivative, we can rewrite the conditional acceleration as follows. 

\begin{fact}
For any $x\sim\pdata$ and $\epsilon\sim\pprior$, we have
$
a_t = \frac{\d^2 \alpha_t}{\d t^2} x + \frac{\d^2 \beta_t}{\d t^2} \epsilon.
$
\end{fact}

Next, we define the marginal acceleration for the expected acceleration field for all possible samples.
\begin{definition}[Marginal acceleration]\label{def:marginal_acceleration}
Let $t \in [0,1]$ denote the timestep. The trajectory $z_t \in \R^d$ is defined in Definition~\ref{def:trajectory}, and the  conditional acceleration $a_t \in \R^d$ defined in Definition~\ref{def:conditional_acceleration}. Then, we define the marginal velocity as follows:
    $
        a(z_t,t) := \E_{a_t \sim p_t(a_t|z_t)} [a_t].
    $
\end{definition}

{\bf Training.}
To train these high-order flow-matching models, the loss function is computed for both velocity and acceleration fields.

\begin{definition}[Second-order flow matching loss]\label{def:second_order_flow_matching_loss}
Let $t \in [0,1]$ denote the timestep. The trajectory $z_t \in \R^d$, marginal velocity $v(z_t,t) \in \R^d$, and marginal acceleration $a(z_t,t) \in \R^d$ are defined in Definitions~\ref{def:trajectory}, \ref{def:marginal_velocity}, and \ref{def:marginal_acceleration}, respectively. To learn the marginal velocity and acceleration fields, neural networks $u_{1,\theta_1}$ and  $u_{2,\theta_2}$ are trained to minimize the following loss function:
    $
        L_{\mathrm{SO}}(\theta) :=  \E_{t,z_t} [\| u_{1,\theta_1}(z_t,t) - v_t(z_t,t) \|^2_2]  + \E_{t,z_t} [\| u_{2,\theta_2}(z_t,t) - a_t(z_t,t) \|^2_2],
    $
    where $t \sim {\sf Uniform}[0,1]$ and $z_t \sim p_t(z_t)$.
\end{definition}

Similar to first-order flow matching models, the original flow matching loss is still intractable. Thus, it is more practical to compute the following alternative loss function.

\begin{definition}[Conditional second-order flow matching loss]\label{def:conditional_second_order_flow_matching_loss}
Let $t \in [0,1]$ denote the timestep. The trajectory $z_t \in \R^d$, conditional velocity $v_t \in \R^d$, and conditional acceleration $a(z_t,t) \in \R^d$ are defined in Definitions~\ref{def:trajectory}, \ref{def:conditional_velocity}, and \ref{def:conditional_acceleration}, respectively. To learn the conditional velocity and acceleration fields, neural networks $u_{1,\theta_1}$ and  $u_{2,\theta_2}$ are trained to minimize the following loss function:
    $
         L_{\mathrm{CSO}}(\theta) := \E_{t,x,\epsilon} [\| u_{1, \theta_1}(z_t,t) - v_t\|_2^2] 
         + \E_{t,x,\epsilon} [\| u_{2, \theta_2}(z_t,t) - a_t\|_2^2],
    $
    where $t \sim {\sf Uniform}[0,1],x\sim \pdata(x),$ and $ \epsilon \sim {\cal N}(\mu,\sigma^2 I_d)$.
\end{definition}

Next, we present an equivalence result between the second-order flow matching loss and the conditional second-order flow matching loss.
\begin{fact}[Equivalence of second-order flow matching loss]
     Up to a constant independent of $\theta$, the second-order flow matching loss $L_{\mathrm{SO}}(\theta)$ in Definition~\ref{def:second_order_flow_matching_loss} and the conditional second-order flow matchingloss $L_{\mathrm{CSO}}(\theta)$ in Definition~\ref{def:conditional_second_order_flow_matching_loss} are equivalent, i.e., $\nabla_{\theta}L_{\mathrm{SO}}(\theta) = \nabla_{\theta}L_{\mathrm{CSO}}(\theta)$. 
\end{fact}
\begin{proof}
    This directly follows from applying Theorem~\ref{thm:equivalence_fm_loss_first_order} on both terms of $L_{\mathrm{SO}}(\theta)$.
\end{proof}

{\bf Sampling.} To sample from second-order flow matching models, we need to consider a second-order ODE and its corresponding $\IVP$.

\begin{definition}[Second-order ODE]\label{def:second_order_ode}
Let $t \in [0,1]$ denote the timestep. The trajectory $z_t \in \R^d$, marginal velocity $v(z_t,t) \in \R^d$, and marginal acceleration $a(z_t,t) \in \R^d$ are defined in Definitions~\ref{def:trajectory}, \ref{def:marginal_velocity}, and \ref{def:marginal_acceleration}, respectively.  Samples can be generated by solving the following ODE:
    $
        \frac{\d^2 z_t}{\d t^2} =a(z_t,t).
    $
\end{definition}

\section{Feasibility of Second-Order \texorpdfstring{$\MeanFlow$}{MeanFlow}}\label{sec:feasibility_meanflow}

In this section, we first provide a formal formulation of the $\MeanFlow$ model in Section~\ref{sec:preliminary:meanflow}, and then introduce our main results on the feasibility of high-order $\MeanFlow$ in Section~\ref{sec:result_second_order_feas}.

\subsection{\texorpdfstring{$\MeanFlow$}{MeanFlow}}\label{sec:preliminary:meanflow}

Recently, $\MeanFlow$~\cite{gdb+25} emerges as a promising variant of flow matching models, which exhibits both sampling efficiency and simple implementation. 

{\bf Average Velocity.} 
$\MeanFlow$~\cite{gdb+25} introduces a concept: the average velocity, which stands in contrast to the instantaneous velocity modelled in traditional Flow Matching. While Flow Matching's velocity describes motions at a specific moment, the average velocity captures the overall movement between two points in time. 
\begin{definition}[Average velocity, implicit in page 3 of~\cite{gdb+25}]\label{def:average_velocity}
    Let $t,r \in [0,1]$ denote two different timesteps. The marginal velocity $v(z_\tau,\tau)$ is defined in Definition~\ref{def:marginal_velocity}. Then, we define the average velocity between $t$ and $r$ as follows:
\begin{align}\label{eq:average_velocity_int}
    \ov{v}(z_t,r,t) := \frac{1}{t-r} \int_r^t v(z_\tau,\tau) \d \tau.
\end{align}
\end{definition}

{\bf Consistency Condition.} The average velocity satisfies an important consistency condition, which ensures the efficient sampling of the training flow matching models~\cite{sdcs23,sd24,gdb+25}. We first introduce the definition of consistency functions.

\begin{definition}[Consistency function, implicit in page 3 of~\cite{sdcs23}]\label{def:consistency_function_new}
    Let $\delta \in (0,1)$ be a small constant. For any timestep $t \in [0,1]$ and trajectory $z_t \in \R^d$ defined in Definition~\ref{def:trajectory}, we say a vector-valued function $f: \R^d\times [0,1] \to \R^d$ is a consistency function if it meets the conditions below:
    \begin{itemize}
        \item {\bf Self-Consistency.} For any $t_1,t_2 \in [\delta, 1]$, we have $f(z_{t_1},t_1) = f(z_{t_2},t_2)$.
        \item {\bf Boundary Condition.} $f(z_{\delta}, \delta) = z_\delta$.
    \end{itemize}
\end{definition}

The consistency condition in $\MeanFlow$ generalizes the definition of the consistency function, which results in the following definition.

\begin{definition}[Generalized consistency function, implicit in pages 3 of~\cite{gdb+25}]\label{def:gen_consis_func}
    For two arbitrary timesteps $t,r \in [0,1]$ ($r \leq t$) and any time-dependent vector trajectory $z_t \in \R^d$, we say a vector-valued function $f: \R^d \times [0,1] \times [0,1] \to \R^d$ is a generalized consistency function for $z_t$ if it meets the conditions below:
     \begin{itemize}
        \item {\bf Self-Consistency.} For any $s \in [r, t]$, we have $(t-r)f(z_t, r, t) = (s-r)f(z_s,r, s) + (t-s)f(z_t,s,t)$.
        \item {\bf Boundary Condition.} $f(z_r,r,r) = z_r$ and $f(z_t,t,t) = z_t$.
    \end{itemize}
\end{definition}

\begin{remark}
    The consistency condition of $\MeanFlow$ in Definition~\ref{def:gen_consis_func} generalizes the original consistency condition in Definition~\ref{def:consistency_function_new} by extending the trajectory $z_t$ to any time-dependent vector-valued function. The generalized consistency is defined on an arbitrary interval $[r,t]$ instead of fixed $[\delta, 1]$. 
\end{remark}

Next, we show that the average velocity in $\MeanFlow$ satisfies the generalized consistency condition, which means that its sampling can be accelerated. We first show that the boundary condition is satisfied. 

\begin{lemma}[Boundary condition of average velocity, informal version of Lemma~\ref{lem:boundary_conditions_avg_v:formal}]\label{lem:boundary_conditions_avg_v}

The relationship between the average velocity $\bar{v}$ (Definition~\ref{def:average_velocity}) and the marginal velocity $v$ (Definition~\ref{def:marginal_velocity}) is established by the boundary condition in Definition~\ref{def:gen_consis_func} as follows:
$
    \ov{v}(z_r,r,r) = v(z_r,r)~~\mathrm{and}~~\ov{v}(z_t,t,t) = v(z_t,t).
$
\end{lemma}

Next, we show that the self-consistency condition also holds for the average velocity.

\begin{lemma}[Consistency constraint of average velocity, informal version of Lemma~\ref{lem:consistency_constraint_avg_v:formal}]\label{lem:consistency_constraint_avg_v}
Let $r,t$ denote two different timesteps. The timestep $s$ is any intermediate time between $r$ and $t$. The marginal velocity $v$ is defined in Definition~\ref{def:marginal_velocity}. The average velocity $\ov{v}$ defined in Definition~\ref{def:average_velocity} satisfies an additive consistency constraint:
$
    (t-r)\ov{v}(z_t, r, t) = (s-r)\ov{v}(z_s, r, s) + (t-s)\ov{v}(z_t, s, t).
$
\end{lemma}

Combining both lemmas above, we can conclude that the average velocity satisfies the generalized consistency condition.
\begin{theorem}[Consistency of average velocity]
    The average velocity $\ov{v}(z_t,r,t)$ defined in Definition~\ref{def:average_velocity} is a generalized consistency function of $v(z_t, t)$ in Definition~\ref{def:marginal_velocity}.
\end{theorem}
\begin{proof}
    This directly follows from Lemma~\ref{lem:boundary_conditions_avg_v} and Lemma~\ref{lem:consistency_constraint_avg_v}.
\end{proof}

    

{\bf Training.} To train the $\MeanFlow$ models, we apply a similar loss function as the vanilla flow-batching models. 

\begin{definition}[Ideal first-order \texorpdfstring{$\MeanFlow$}{MeanFlow} loss]\label{def:ideal_first_order_meanflow_loss}
    $
        L_{\mathrm{FOM}}^{*}(\theta) := \E_{t, z_t}[\| u_{1,\theta_1}(z_t,r,t) - \ov{v}(z_t, r, t) \|_2^2],
    $
    where $t \sim \mathsf{Uniform}[0,1]$, $r \sim \mathsf{Uniform}[0,1]$ and $z_t \sim p_t(z_t)$. Here $\mathrm{FOM}$ denotes first-order $\MeanFlow$.
\end{definition}

\begin{remark}
    
The $\MeanFlow$ model aims to approximate the average velocity $\ov{v}(z_t,r,t)$ defined in Definition~\ref{def:average_velocity} with a neural network $u_{\theta}(z_t,r,t)$. This strategy offers a distinct advantage: provided an accurate approximation, we can reconstruct the entire flow trajectory with evaluation of $u_{\theta}(\epsilon,0,1)$. Empirically, this makes the approach significantly more suitable for single or few-step generation, as it eliminates the need for explicit time integral approximation during inference, which is a challenge inherent in models of instantaneous velocity. 

\end{remark}
\begin{remark}
    
Nevertheless, directly using the average velocity as an objective function for training is impractical because it necessitates integral evaluation during the training phase. The core innovation of meanflow lies in transforming the definitional equation of the average velocity to an optimization target that can be effectively trained, even with access to only instantaneous velocity.
\end{remark}

Next, we present a practical loss function for $\MeanFlow$, which carefully tackles the computation of mean velocities and conditional expectations.

\begin{definition}[First-order $\MeanFlow$ loss, implicit in page 5 on~\cite{gdb+25}]\label{def:first_order_meanflow_loss}

Let $\mathrm{sg}(\cdot)$ denote the stopping gradient operation. Let $t,r \in [0,1]$ denote the timestep. $z_t$ is defined in Definition~\ref{def:trajectory}. A neural network $u_{1,\theta_1}$ is trained to minimize the following loss function:
$
    L_{\mathrm{FOM}}(\theta) := \E_{r,t,x,\epsilon} [\| u_{1,\theta_1}(z_t,r,t) - \mathrm{sg}(\ov{v}(z_t,r,t)) \|_2^2],
$
where $t \sim {\sf Uniform}[0,1], r \sim {\sf Uniform}[0,1], x\sim \pdata(x),\epsilon \sim {\cal N}(\mu,\sigma^2 I_d)$, and the mean velocity can be computed as:
\begin{align}\label{eq:average_velocity_diff}
    \ov{v}(z_t,r,t) = v(z_t,t) - (t-r)\frac{\d}{\d t}\ov{v}(z_t,r,t).
\end{align}

\end{definition}
\begin{remark}
    To facilitate optimization and prevent ``double backpropagation'' through the Jacobian-vector product, a stop-gradient (sg) operation is applied to $\ov{v}(z_t,r,t)$ in the loss function, consistent with previous works~\cite{sdcs23,sd24,gpl+24,ls25,fhla25}.
\end{remark}

The loss function in Definition~\ref{def:first_order_meanflow_loss} has two desirable properties, which ensure its effectiveness and computational efficiency. First, we show the effectiveness result. 

\begin{theorem}[Equivalence of \texorpdfstring{$\MeanFlow$}{MeanFlow} loss functions, informal version of Theorem~\ref{thm:equi_meanflow_loss:formal}]\label{thm:equi_meanflow_loss}
    The first-order $\MeanFlow$ loss function $L_{\mathrm{FOM}}(\theta)$ in Definition~\ref{def:first_order_meanflow_loss} is equivalent to the ideal first-order $\MeanFlow$ loss function $L_{\mathrm{FOM}}^*(\theta)$ in Definition~\ref{def:ideal_first_order_meanflow_loss}. 

\end{theorem}

Next, we show that the first-order $\MeanFlow$ loss can be computed efficiently with JVPs. 

\begin{theorem}
[Efficiency of $\MeanFlow$ loss, informal version of Theorem~\ref{thm:eff_meanflow_loss:formal}]\label{thm:eff_meanflow_loss}
    The loss function $L_{\mathrm{FOM}}(\theta)$ in Definition~\ref{def:first_order_meanflow_loss} can be evaluated efficiently with the conditional velocity $v_t$ in Definition~\ref{def:conditional_velocity}  and Jacobian-Vector Products (JVPs) of the neural network $u_{1, \theta_1}$, i.e.,
    $
        \ov{v}(z_t,r,t) = v_t - (t-r)(v_t \partial_z u_\theta + \partial_t u_\theta).
    $
\end{theorem}

\begin{remark}
In contemporary programming libraries, the JVP computations in Theorem~\ref{thm:eff_meanflow_loss} can be performed efficiently using JVP interfaces, such as \texttt{torch.func.jvp} in PyTorch or \texttt{jax.jvp} in JAX, which ensures convenient implementation and fast computation.     
\end{remark}

{\bf Sampling.} The sampling process of the first-order $\MeanFlow$ model simply follows the same way as the iterative steps in Fact~\ref{fac:first_order_euler}, which is the same as the first-order flow matching. Since the average velocity satisfies the generalized consistency function, it allows a fast one-step sampling, i.e.,
    $z_0 = z_1 - \ov{v}(z_1,0,1),$
where $z_1\sim \pprior$ is sampled from the prior distribution.

\subsection{Main Results on Second-Order \texorpdfstring{$\MeanFlow$}{MeanFlow}}\label{sec:result_second_order_feas}
Building on the concept of average velocity, we introduce average acceleration $\ov{a}$, which extends the vanilla $\MeanFlow$ to a high-order case. 

\begin{definition}\label{def:average_acceleration}
Let $t,r$ denote two different timesteps. For any marginal acceleration $a(z_{t},t)$ defined in Definition~\ref{def:marginal_acceleration}, we define the average acceleration as follows:
    \begin{align}\label{eq:average_acceleration_int}
    \ov{a} (z_t,r,t) := \frac{1}{t-r}\int_r^t a(z_{\tau},\tau) \d \tau
\end{align}
\end{definition}

\paragraph{Consistency Condition.} 
We first show our key result on the consistency conditions. We begin by proving that the boundary condition is satisfied for the average acceleration. 

\begin{lemma}[Boundary conditions of average acceleration, informal version of Lemma~\ref{lem:boundary_conditions_avg_a:formal}]\label{lem:boundary_conditions_avg_a}
Given the average acceleration $\ov{a}$ defined in Definition~\ref{def:average_acceleration} and the marginal acceleration $a$ defined in Definition~\ref{def:marginal_acceleration}, $\ov{a}$ converges to $a$ as the time interval shrinks to zero. Specifically:
    $\ov{a}(z_r,r,r) = a(z_r,r)~~\mathrm{and}~~\ov{a}(z_t,t,t) = a(z_t,t).$
\end{lemma}

Similar to the first-order $\MeanFlow$ results, the consistency constraint also holds for the average acceleration.

\begin{lemma}[Consistency constraint of average acceleration, informal version of Lemma~\ref{lem:consistency_constraint_avg_a:formal}] \label{lem:consistency_constraint_avg_a}
Let $r, s,$ and $t$ be three distinct time steps, such that $s$ lies between $r$ and $t$. Given the marginal acceleration $a$ (as defined in Definition~\ref{def:marginal_acceleration}), the average acceleration $\ov{a}$ (as defined in Definition~\ref{def:average_acceleration}) satisfies the following additive consistency constraint:
    $(t-r)\ov{a}(z_t, r, t) = (s-r)\ov{a}(z_s, r, s) + (t-s)\ov{a}(z_t, s, t).$
\end{lemma}

Combining both lemmas above, we obtain the consistency result for the second-order $\MeanFlow$.

\begin{theorem}[Consistency of second-order $\MeanFlow$]\label{thm:consistency_mean_acceleration}
    The average acceleration $\ov{a}(z_t,r,t)$ defined in Definition~\ref{def:average_acceleration} is a generalized consistency function of $a(z_t, t)$ in Definition~\ref{def:marginal_acceleration}.
\end{theorem}
\begin{proof}
    This directly follows from Lemma~\ref{lem:boundary_conditions_avg_a} and Lemma~\ref{lem:consistency_constraint_avg_a}.
\end{proof}

{\bf Training.} 
Similar to the original $\MeanFlow$, we show that the second-order extension also enjoys effective and efficient training. First, we show the ideal form of the second-order $\MeanFlow$ loss, which matches the original flow matching loss function but is intractable in practice.

\begin{definition}[Ideal second-order $\MeanFlow$ loss]\label{def:ideal_second_order_meanflow_loss}
    \begin{align*}
    L_{\mathrm{SOM}}^{*}(\theta) := & ~ \E_{t,r,z_t}[\| u_{1,\theta_1}(z_t,r,t) - \ov{v}(z_t,r,t) \|_2^2] \\
    & ~ + \E_{t,r,z_t}[\| u_{2,\theta_2}(z_t,r,t) - \ov{a}(z_t,r,t) \|_2^2],
    \end{align*}
    where $t\sim\mathsf{Uniform}[0,1]$,$r\sim\mathsf{Uniform}[0,1]$, and $z_t \sim p_t(z_t)$.
\end{definition}

The intractability of the ideal second-order $\MeanFlow$ loss has necessitated the following loss function:

\begin{definition}[Second-order $\MeanFlow$ loss]\label{def:second_order_meanflow_loss}
Let $\mathrm{sg}(\cdot)$ denote the stopping gradient operation. Let $t,r \in [0,1]$ denote the timestep. $z_t$ is defined in Definition~\ref{def:trajectory}. Two neural networks $u_{1,\theta_1}$ and $u_{2,\theta_2}$ are trained to minimize the following loss function:
        $L_{\mathrm{SOM}}(\theta) 
        := \E_{t,r,x,\epsilon}[\| u_{1,\theta_1}(z_t,r,t) - \ov{v}(z_t,r,t) \|_2^2]
        +  \E_{t,r,x,\epsilon}[\| u_{2,\theta_2}(z_t,r,t) - \ov{a}(z_t,r,t) \|_2^2],$
    where $t\sim\mathsf{Uniform}[0,1]$,$r\sim\mathsf{Uniform}[0,1]$, $x\sim \pdata(x),\epsilon \sim {\cal N}(\mu,\sigma^2 I_d)$, and the mean velocity and acceleration can be computed as:
    $\ov{v}(z_t,r,t) =  v(z_t,t) - (t-r)\frac{\d}{\d t}\ov{v}(z_t,r,t), 
        \ov{a}(z_t,r,t) =  a(z_t,t) - (t-r)\frac{\d}{\d t}\ov{a}(z_t,r,t).$

\end{definition}

The loss function above has two desirable properties. First, we show that the tractable second-order $\MeanFlow$ loss function is equivalent to the ideal loss function.

\begin{theorem}[Effectivenss of second-order $\MeanFlow$, informal version of Theorem~\ref{thm:equi_2nd_meanflow_loss:formal}]\label{thm:equi_2nd_meanflow_loss}
    The loss function $L_{\mathrm{SOM}}(\theta)$ in Definition~\ref{def:second_order_meanflow_loss} is equivalent to the ideal $\MeanFlow$ loss function $L_{\mathrm{SOM}}^*(\theta)$ in Definition~\ref{def:ideal_second_order_meanflow_loss}. 
\end{theorem}

Next, we show that the second-order $\MeanFlow$ loss can be computed efficiently with JVPs.

\begin{theorem}[Efficiency of second-order $\MeanFlow$, informal version of Theorem~\ref{thm:efficiency_2nd_meanflow_loss:formal}]\label{thm:efficiency_2nd_meanflow_loss}
    The loss function $L_{\mathrm{SOM}}(\theta)$ in Definition~\ref{def:second_order_meanflow_loss} can be evaluated efficiently with the conditional velocity $v_t$ and conditional acceleration $a_t$ in and Jacobian-Vector Products (JVPs) of the neural networks $u_{1, \theta_1}$ and $u_{2, \theta_2}$, i.e.,
        $\ov{v}(z_t,r,t) = v_t - (t-r)(v_t \partial_z u_{1, \theta_1} + \partial_t u_{1, \theta_1}), 
         \ov{a}(z_t,r,t) = a_t - (t-r)(a_t \partial_z u_{2, \theta_2} + \partial_t u_{2, \theta_2}).$
\end{theorem}

\section{Circuit Complexity of Second-order \texorpdfstring{$\MeanFlow$}{MeanFlow}} \label{sec:circuit_complexity}
In this section, we prove the circuit complexity of $\MeanFlow$ and its second-order version under mild assumptions.

\subsection{\texorpdfstring{$\MeanFlow$}{MeanFlow} Formation}
In this section, we provide the formal definitions for $\MeanFlow$'s sampling.

\begin{definition}[$\MeanFlow$ $T$-step sampling] \label{def:meanflow_sampling_vit}
    Given $T \in \mathbb{Z}_+$ as the total iteration of sampling, and arbitrary $\{t_i\}_{i=1}^T$ that satisfy $t_{i} > t_{i+1}$ for all $i \in \{1,\ldots,T-1\}$ as timestep scheduler. Let $Z_1 \in \mathbb{F}_p^{n \times d}$ become the initial feature map for $\MeanFlow$. Then according to Fact~\ref{fac:first_order_euler}, the sampling process of $\MeanFlow$ is given by
        $ {\sf Z}_{t_i} :=  {\sf Z}_{t_{i-1}} + (t_{i} - t_{i-1}){\sf ViT}({\sf Z}_{t_{i-1}} ~ || ~ t_{i} {\bf 1}_{n} ~ || ~ t_{i-1} {\bf 1}_{n})_{*,1:d} \in \R^{n},$
    where $||$ denotes the concatenation operation.
    We use
        ${\sf MF}(Z_{t_1},t_1, \cdots, t_T) := Z_{t_T}$
    to denote a $T$-step sampling of $\MeanFlow$ as a function.
\end{definition}

\subsection{Second-order \texorpdfstring{$\MeanFlow$}{MeanFlow} Formation}
In this section, we provide the formal definitions of Second-order $\MeanFlow$'s sampling.

\begin{definition}[Second-order $\MeanFlow$ $k$-step sampling]\label{def:2nd_meanflow_sampling_vit}
    Given $T \in \mathbb{Z}_+$ as the total iteration of sampling, and arbitrary $\{t_i\}_{i=1}^T$ that satisfy $t_{i} > t_{i+1}$ for all $i \in \{1,\ldots,T-1\}$ as timestep scheduler. Let $Z_0 \in \mathbb{F}_p^{n \times d}$ become the initial feature map for $\MeanFlow$. Then according to Fact~\ref{fac:second_order_euler}, the sampling process of Second-order $\MeanFlow$ is given by
        $ {\sf Z}_{t_i} := {\sf Z}_{t_{i-1}} + (t_{i} - t_{i-1}) {\sf ViT}_1({\sf Z}_{t_{i-1}} ~ ||~  t_{i} {\bf 1}_{n} ~ ||~ t_{i-1} {\bf 1}_{n})_{*,1:d} 
         + \frac{1}{2}(t_{i} - t_{i-1})^2 {\sf ViT}_2({\sf Z}_{t_{i-1}} ~ || ~ t_{i} {\bf 1}_{n} ~ || ~ t_{i-1} {\bf 1}_{n})_{*,1:d}
        \in \R^{n \times d},$
    where $\sf ViT_1$ and $\sf ViT_2$ are two ViTs, and $||$ denotes the concatenation operation.
    We use
        ${\sf SMF}(Z_{t_1},t_1, \cdots, t_T) := Z_{t_T}$
    to denote a $T$-step sampling of Second-order $\MeanFlow$ as a function.
\end{definition}

\subsection{Circuit Complexity of ViT}

In this section, we provide the circuit complexity of ViT.

We first introduce a the circuit complexity class of matrix multiplication, which is a crucial component for our work.

\begin{lemma}[Matrix Multiplication  belongs to $\mathsf{TC}^0$ class, Lemma 4.2 in \cite{cll+24}] \label{lem:matrix_multiplication_tc0}
Let $M \in \mathbb{F}_{p}^{n_1 \times d}$ and $M \in \mathbb{F}_p^{d \times n_2}$. If precision $p \leq \poly(n)$, $n_1, n_2 \leq \poly(n)$, and $d \leq n$. Then $MN$ can be computed by a uniform threshold circuit of depth $( d_{\rm std}+d_{\oplus} )$ and size $\poly(n)$.
\end{lemma}

Then, we demonstrate that the computation of attention matrix belongs to $\mathsf{TC}^0$ class.

\begin{lemma}[Attention matrix computation belongs to $\mathsf{TC}^0$ class, Lemma 5.3 in \cite{kll+25}]\label{lem:attn_tc0}
Assume the precision $p \leq \poly(n)$, then we can use a size bounded by $\poly(n)$ and constant depth $3(d_{\rm std} + d_{\oplus}) + d_{\rm exp}$ uniform threshold circuit to compute the attention matrix $A$ defined in Definition~\ref{def:attn_matrix}.
\end{lemma}

Next, we outline the circuit complexity of the computation of MLP layers.
\begin{lemma}[MLP computation falls within $\mathsf{TC}^0$ class, Lemma 4.5 of \cite{cll+24}]\label{lem:mlp_tc0}
    Assume the precision $p \leq \poly(n)$. Then, we can use a size bounded by $\poly(n)$ and constant depth $2d_\mathrm{std} + d_{\oplus}$ uniform threshold circuit to simulate the MLP layer in Definition~\ref{def:mlp}.
\end{lemma}

In this work, LayerNorm is also a key component, which is used multiple times.
\begin{lemma}[LN computation falls within $\mathsf{TC}^0$ class, Lemma 4.6 of \cite{cll+24}]\label{lem:layer_tc0}
    Assume the precision $p \leq \poly(n)$, then we can use a size bounded by $\poly(n)$ and constant depth $5d_\mathrm{std} + 2d_{\oplus} + d_\mathrm{sqrt}$ uniform threshold circuit to simulate the Layer-wise Normalization layer defined in Definition~\ref{def:layer_norm}.
\end{lemma}

Now we already have the circuit complexity for necessary components. Then we start to prove the circuit complexity of ViT formally. We first provide the circuit complexity of the computation of $Y_i$ of ViT.

\begin{lemma}[$Y_i$ of ViT computation in $\mathsf{TC}^0$, informal version of Lemma~\ref{lem:vit_y_tc0:formal}]\label{lem:vit_y_tc0}
    Assume the precision $p \leq \poly(n)$, then for a $m$ layer ViT, we can use a size bounded by $\poly(n)$ and constant depth $m(9d_{\rm std} + 5d_{\oplus} + d_{\rm sqrt} + d_{\rm exp})$ uniform threshold circuit to simulate the computation of $Y_i$ defined in Definition~\ref{def:vit_with_residual}.
\end{lemma}

Then we prove the computation of $X_i$ of ViT belongs to $\mathsf{TC}^0$.
\begin{lemma}[$X_i$ of ViT computation in $\mathsf{TC}^0$, informal version of Lemma~\ref{lem:vit_x_tc0:formal}] \label{lem:vit_x_tc0}
    Assume the precision $p \leq \poly(n)$, then for a $m$ layer ViT, we can use a size bounded by $\poly(n)$ and constant depth $m(8d_{\rm std} + 3d_{\oplus} + d_{\rm sqrt})$ uniform threshold circuit to simulate the computation of $X_i$ defined in Definition~\ref{def:vit_with_residual}.
\end{lemma}

Combining previous two results together, we arrive at the circuit complexity of the complete ViT computation.
\begin{lemma}[ViT computation in $\mathsf{TC}^0$, informal version of Lemma~\ref{lem:vit_tc0:formal}]\label{lem:vit_tc0}
    Assume the number of transformer layers $m = O(1)$. Assume the precision $p \leq \poly(n)$. Then, we can apply a uniform threshold circuit to simulate the ViT defined in Definition~\ref{def:vit_with_residual}. The circuit has size $\poly(n)$ and $O(1)$ depth.
\end{lemma}

\subsection{Circuit Complexity of \texorpdfstring{$\MeanFlow$}{MeanFlow} with Euler solver}
In this section, we show the circuit complexity of sampling vanilla $\MeanFlow$ and its second-order version with Euler solver both belong to $\mathsf{TC}^0$.

We first prove the original $\MeanFlow$ sampling with $T$-step Euler solver belongs to $\mathsf{TC}^0$.

\begin{theorem}[$\MeanFlow$ sampling with $T$-step Euler solver belongs to $\mathsf{TC}^0$, informal version of Theorem~\ref{thm:meanflow_sampling_tc0:formal}] \label{thm:meanflow_sampling_tc0}
    Given $T \in \mathbb{Z}_+$ as the total iteration of sampling, and arbitrary $\{t_i\}_{i=1}^T$ that satisfy $t_{i} > t_{i+1}$ for all $i \in \{1,\ldots,T-1\}$ as timestep scheduler. 
    Assume the precision $p \leq \poly(n)$, the number of transformer layers $m = O(1)$, and $T = O(1)$.
    Then we can use a size bounded by $\poly(n)$ and constant depth $T(m(17d_{\rm std} + 8d_{\oplus} + 2d_{\rm sqrt} + d_{\rm exp}) + 5d_{\rm std})$ uniform threshold circuit to simulate the $\MeanFlow$ $T$-step sampling defined in Definition~\ref{def:meanflow_sampling_vit}.
\end{theorem}

We then prove the sampling circuit complexity of Second-order $\MeanFlow$ with $T$-step Euler solver belongs to $\mathsf{TC}^0$.
\begin{theorem}[Second-order $\MeanFlow$ sampling with $T$-step Euler solver belongs to $\mathsf{TC}^0$, informal version of Theorem~\ref{thm:2nd_thm:meanflow_sampling_tc0:formal}]\label{thm:2nd_thm:meanflow_sampling_tc0}
    Given $T \in \mathbb{Z}_+$ as the total iteration of sampling, and arbitrary $\{t_i\}_{i=1}^T$ that satisfy $t_{i} > t_{i+1}$ for all $i \in \{1,\ldots,T-1\}$ as timestep scheduler. 
    Assume the precision $p \leq \poly(n)$, the number of transformer layers $m = O(1)$, and $T = O(1)$.
    Then we can use a size bounded by $\poly(n)$ and constant depth $2T(m(17d_{\rm std} + 8d_{\oplus} + 2d_{\rm sqrt} + d_{\rm exp}) + 5d_{\rm std}) + Td_{\rm std}$ uniform threshold circuit to simulate the $\MeanFlow$ $T$-step sampling defined in Definition~\ref{def:2nd_meanflow_sampling_vit}.
\end{theorem}

\section{Provably Efficient Criteria}\label{sec:provably_efficiency}
In this section, we first provide an approximate attention computation in Section~\ref{sec:provably_efficiency:appro_atten}.
Then, we present a running time analysis for Second-Order $\MeanFlow$ in Section~\ref{sec:provably_efficiency:running_time_orginal}.
We also present the running time analysis for its fast version in Section~\ref{sec:provably_efficiency:running_time_fast}.
Next, we analyze the error propagation of the fast Second-Order $\MeanFlow$ in Section~\ref{sec:provably_efficiency:error_propagation_analysis}.
Finally, we show the existence result of a fast algorithm that computes the second-order $\MeanFlow$ architecture in Section~\ref{sec:provably_efficiency:quadratic_time_algorithm}.

\subsection{Approximate Attention Computation}\label{sec:provably_efficiency:appro_atten}
In this section, we introduce approximate attention computation, which accelerates the attention layer's computation.
\begin{definition}[Approximate attention computation, Definition 1.2 in~\cite{as23}]\label{def:appro_atten}
Let $n, d > 0$ be positive constants representing the number of tokens and the dimension of the embeddings, respectively. We are given three matrices: the Query matrix $Q \in \R^{n \times d}$, the Key matrix $K \in \R^{n \times d}$, and the Value matrix $V \in \R^{n \times d}$. These matrices are guaranteed to have a bounded infinity norm:
$\|Q\|_{\infty} \leq R,\|K\|_{\infty} \leq R, \|V\|_{\infty} \leq R$
for some known constant $R> 0$.

For any attention layer ${\sf Attn}(Q,K,V)$, we use an approximate attention layer, denoted by ${\sf AAttC}(n,d,R,\delta)$, which produces an output matrix $O \in \R^{n \times d}$ that approximates the true output with a guaranteed error bound $\delta = 1 / \poly(n)$:
\begin{align*}
    \|O - {\sf Attn}(Q,K,V)\|_{\infty} \leq 1 / \poly(n).
\end{align*}

\end{definition}

The next lemma specifies the time complexity for the ${\sf AAttC}$ method.

\begin{theorem}[Time complexity of approximate attention computation, Theorem 1.4 of~\cite{as23}]\label{thm:fast_atten}
We define ${\sf AAttC}$ in Definition~\ref{def:appro_atten}. Let the parameters for ${\sf AAttC}$ be set as follows: an embedding dimension of $d=O(\log n)$, $R=\Theta(\sqrt{\log n})$, and an approximation tolerance of $\delta=1/\mathrm{poly}(n)$. Based on these conditions, the time complexity is:
    ${\cal T}(n, n^{o(1)}, d) = n^{1+o(1)}.$

\end{theorem}

\subsection{Inference runtime of Second-order \texorpdfstring{$\MeanFlow$}{MeanFlow}}\label{sec:provably_efficiency:running_time_orginal}

This section analyzes the total running time of the inference pipeline for the original Second-order $\MeanFlow$ architecture.

\begin{lemma}[Sampling runtime of Second-order $\MeanFlow$, informal version of Lemma~\ref{lem:runtime_old_2nd_meanflow:formal}]\label{lem:runtime_old_2nd_meanflow}
Consider the original second-order $\MeanFlow$ inference pipeline. The input is a tensor ${\sf X} \in \R^{h \times w \times c}$, where the height $h$ and width $w$ are both equal to $n$, and the number of channels $c$ is on the order of $O(\log n)$. 
The interpolated state at time $t \in [0,1]$ is denoted by ${\sf F}^t$, with ${\sf F}^1$ representing the final state. 
The model architecture consists of attention $(\mathsf{Attn})$, MLP $(\mathsf{MLP}(\cdot, c, d))$, and Layer Normalization $(\mathsf{LN})$ layers. Based on these conditions, the inference time complexity of second-order $\MeanFlow$ is bounded by $O(n^{4+o(1)})$.
    
\end{lemma}

\subsection{Inference runtime of Fast Second-order \texorpdfstring{$\MeanFlow$}{MeanFlow}}\label{sec:provably_efficiency:running_time_fast}
This section applies the conclusions of \cite{as23} to the Second-order $\MeanFlow$ architecture. Specifically, we replace each of its Attention modules with the Approximate Attention defined in Definition~\ref{def:appro_atten}.

\begin{lemma}[Sampling runtime of fast Second-order $\MeanFlow$, informal version of Lemma~\ref{lem:runtime_2nd_meanflow_fast:formal}]\label{lem:runtime_2nd_meanflow_fast}
Consider the fast second-order $\MeanFlow$ inference pipeline. It takes an input tensor ${\sf Z}_1 \in \R^{h \times w \times c}$, where the height $h=n$, width $w=n$, and the number of channels $c = O(\log n)$. The model architecture consists of attention $(\mathsf{AAttC})$, MLP $(\mathsf{MLP}(\cdot, c, d))$, and Layer Normalization $(\mathsf{LN})$ layers.
The interpolated state at time $t \in [0,1]$ is denoted by ${\sf Z}_t$, with ${\sf Z}_0$ as the final state. Based on these conditions, the inference time complexity of fast second-order $\MeanFlow$ is bounded by  $O(n^{2+o(1)})$.
\end{lemma}

\subsection{Analysis of Error Propagation}\label{sec:provably_efficiency:error_propagation_analysis}
This section shows an error analysis for the approximate attention applied to the second-order $\MeanFlow$ model.

We conduct the error analysis between approximate attention computation in Definition~\ref{def:appro_atten} and vanilla attention. We begin by analyzing the errors for the attention matrices.

\begin{lemma}[Error analysis of ${\sf AAttC}({\sf Z'}_1)$ and ${\sf Attn}({\sf Z}_1)$, Lemma B.4 in~\cite{kll+25}]\label{lem:error_analysis_of_attention}

Let ${\sf AAttC}({\sf Z'}_1)$ be an input tensor and ${\sf Attn}({\sf Z}_1)$ be its approximation, satisfying an element-wise error bound $\| {\sf Z'}_1 - {\sf Z}_1 \| \leq \epsilon$ for some $\epsilon \in (0,0.1)$. Let ${\sf Attn}(\cdot)$ denote the standard attention layer defined in Definition~\ref{def:attn_matrix} and ${\sf AAttC}(\cdot)$ represent the approximate attention layer defined in Definition~\ref{def:appro_atten}, which utilizes a polynomial $f$ of degree $g$ constructed from low-rank matrices $U,V\in \R^{hw \times k}$. Assume the entries of the query, key, and value matrices are bounded in magnitude by a constant $R > 1$. Then, the element-wise error between the approximated attention output on the approximated input and the standard attention output on the original input is bounded as follows:
    $\| {\sf AAttC}({\sf Z'}_1) - {\sf Attn}({\sf Z}_1)\|_{\infty} \leq O(kR^{g+1}c)\cdot \epsilon,$
where we extend the $\ell_{\infty}$ norm to apply tensors.
\end{lemma}

Next, we analyze the MLP layer, which is part of the deep architecture used to compute the $\MeanFlow$ layers.

\begin{lemma}[Error analysis of MLP layer, Lemma C.3 of \cite{gkl+25}]\label{lem:error_analysis_mlp}
Let ${\sf MLP}(\cdot,c,d)$ be the MLP layer as defined in Definition~\ref{def:mlp}. Let ${\sf AAttC}({\sf Z'}_1)$ be an input tensor and ${\sf Attn}({\sf Z}_1)$ be its approximation, satisfying an element-wise error bound $\| {\sf Z'}_1 - {\sf Z}_1 \| \leq \epsilon$ for some $\epsilon \in (0,0.1)$. Assume the entries of the MLP's internal weight matrices are bounded in magnitude by a constant $R > 1$. Then, the element-wise error between the MLP outputs for the approximated and original inputs is bounded as follows:
    $\| {\sf MLP}({\sf Z'}_1) - {\sf MLP}({\sf Z}_1)\|_{\infty} \leq cR\epsilon,$
where we abuse the $\ell_{\infty}$ norm in its tensor form for clarity.
\end{lemma}

Then, we compute the error bound for the $\MeanFlow$ layer under fast attention computation.
\begin{lemma}[Error bound between fast second-order $\MeanFlow$ layer and second-order $\MeanFlow$ layer, informal version of Lemma~\ref{lem:error_bound_fast_2nd_meanflow_layers:formal}]\label{lem:error_bound_fast_2nd_meanflow_layers}
For the error analysis of the second-order $\MeanFlow$ Layer, we establish the following conditions. 
Let ${\sf Z}_1 \in \R^{h \times w \times c}$ be the input tensor and let ${\sf Z}_1' \in \R^{h \times w \times c}$ be its approximation, such that the approximation error is bounded by $\|{\sf Z}_1 - {\sf Z}_1'\|_{\infty} \leq \epsilon$ for some small constant $\epsilon > 0$. 

The analysis considers interpolated inputs ${\sf Z}_t, {\sf Z}_{\mathrm{fast},t} \in \R^{h \times w \times c}$ over a time step $t \in [0, 1]$. These inputs are processed by a standard second-order $\MeanFlow$ layer, ${\sf SMF}(\cdot, \cdot, \cdot)$, and a fast variant, ${\sf SMF}_{\mathrm{fast}}(\cdot, \cdot, \cdot)$,
where the fast variant substitute $\sf Attn$ operation with $\sf AAttC$ (Definition~\ref{def:appro_atten}).
The approximation is achieved via a polynomial $f$ of degree $g$ and low-rank matrices $U, V \in \R^{hw \times k}$. 

We assume all matrix entries are bounded by a constant $R > 1$. Crucially, we also make assumption that the LayerNorm function, ${\sf LN}(\cdot)$ defined in Definition~\ref{def:layer_norm}, does not exacerbate error propagation; that is, if $\|{\sf Z}_1' - {\sf Z}_1\|_{\infty} \le \epsilon$, then it follows that $\|{\sf LN}({\sf Z}_1') - {\sf LN}({\sf Z}_1)\|_{\infty} \le \epsilon$.

Then, we have
\begin{align*}
    \| {\sf SMF}_{\mathrm{fast}}({\sf Z}_{\mathrm{fast},t}, t, r) - {\sf SMF}({\sf Z}_t, t, r) \|_{\infty} \leq O(c^2kR^{g+2}) \epsilon.
\end{align*}
\end{lemma}

\subsection{Almost Quadratic Time Algorithm}\label{sec:provably_efficiency:quadratic_time_algorithm}
This section shows a theorem on the existence of a almost quadratic-time algorithm that approximates the second-order $\MeanFlow$ architecture with guaranteed a additive error bound.

\begin{theorem}[Almost quadratic time algorithm]\label{thm:existence_quadratic_alg}
  Suppose $d = O(\log n)$ and $R = o(\sqrt{\log n})$. Then, there exists an algorithm that can approximate the second-order $\MeanFlow$ architecture with an additive error of at most $1/\poly(n)$ in $O(n^{2+o(1)})$ time.
\end{theorem}

\begin{proof}
    The result follows directly by combining Lemma~\ref{lem:runtime_2nd_meanflow_fast} and Lemma~\ref{lem:error_bound_fast_2nd_meanflow_layers}.
\end{proof}
\section{Discussion} \label{sec:discussion}

In this section, we discuss the motivation of the proposed high-order $\MeanFlow$ models, and show the practical implications of our theoretical results.

{\bf Why High-order $\MeanFlow$ Matters?}
Modern generative models rely on solving differential equations that describe how data transforms over time. In practice, these flows are typically solved using first-order methods like Euler's method, which capture only the immediate direction of change while ignoring higher-order structure.

Consider any smooth trajectory $z(t)$. Taylor's theorem tells us we can expand around time $r$:
\begin{align*}
    z_t = z_r + (t-r) z'_r + \frac{(t-r)^2}{2} z''_r + \frac{(t-r)^3}{6} z'''(\xi),
\end{align*}
where $\xi \in (r,t)$.

First-order methods use only the linear terms:
$
z_t^{(1)} := z_r + (t-r) z'_r,
$
while second-order methods include the quadratic term:
$
    z_t^{(2)} := z_r + (t-r) z'_r + \frac{(t-r)^2}{2} z''r.
$

Then, the approximation errors are:
\begin{itemize}
\item First-order error: 
    \begin{align*}
        E_1 = & ~ z_t - z_t^{(1)} \\
        = & ~ \frac{(t-r)^2}{2} z''_r + \frac{(t-r)^3}{6} z'''(\xi) = O((t-r)^2).
    \end{align*}
\item Second-order error:
    \begin{align*}
        E_2 = z_t - z_t^{(2)} = \frac{(t-r)^3}{6} z'''(\xi) = O((t-r)^3).
    \end{align*}
\end{itemize}

This implies that second-order methods can achieve the same accuracy with larger steps, or better accuracy with the same computational cost. For generative models, this translates to fewer integration steps for sampling, improved numerical stability during training, and richer expressiveness through curvature information.

High-Order $\MeanFlow$ leverages this insight to improve flow-based generative modeling by incorporating second-order derivatives into the flow approximation process.

{\bf Practical Implications of Expressivity (Theorem~\ref{thm:2nd_thm:meanflow_sampling_tc0}).} 
Circuit complexity is a fundamental theoretical tool for analyzing the expressiveness of ML models. For example, Transformers without chain-of-thought prompting~\cite{wws+22} belong to the $\mathsf{TC}^0$ class~\cite{llzm24}, which makes it unable to solve $\mathsf{NC}^1$ hard unless an open conjecture $\mathsf{TC}^0 = \mathsf{NC}^1$ holds~\cite{v99}. In this work, we make a significant extension of this theory to the context of generative modeling, and find that $\MeanFlow$ models with or without high-order flow augmentations belong to the $\mathsf{TC}^0$ class in Theorem~\ref{thm:2nd_thm:meanflow_sampling_tc0}.  

Thus, the key takeaway for practitioners is that, despite the strong modeling capability of high-order mean-flow models in capturing the distribution of image or video data, such architectures may not resolve the inherent limitations of their backbone models, such as ViT~\cite{dbk+20} or DiT~\cite{px23}. We suggest that practitioners may consider expressivity enhancement techniques in such models, such as looped Transformers~\cite{grs+23,ylnp24,df24,sdl+25}, or improved positional encoding techniques~\cite{ysw+25}.

{\bf Practical Implications of Provable Efficient Criteria (Theorem~\ref{thm:existence_quadratic_alg}).} Our key results in Theorem~\ref{thm:existence_quadratic_alg} show that the existence of an algorithm that is quadratic in the image size $n$ to compute the forward process of High-order $\MeanFlow$ models, and what condition is needed for the existence of such efficient approximations. First, this suggests that our high-order $\MeanFlow$ models can be computed efficiently in practice, with desirable inference speed guaranteed in theory. Second, the condition $R = o(\sqrt{\log n})$ highlights that proper normalization may be needed to train this new type of model, since extremely large model weights may adversely affect the approximation error and make such efficient algorithms impossible. Our result extends previous provable efficiency results on LLMs~\cite{as23,as24_iclr,as24_neurips,as25_rank,as25_rope} and offers new insights within the generative modeling setting in vision.
\section{Conclusion} \label{sec:conclusion}

We have presented Second-Order $\MeanFlow$ as a principled generalization of $\MeanFlow$ that models not only average velocities but also average accelerations, thereby enriching the dynamic expressivity of simulation‐free generative flows. Our theoretical analysis demonstrates that (i) average acceleration fields obey a consistency condition enabling single‐step inference, (ii) the entire sampling mechanism resides within the $\mathsf{TC}^0$ circuit complexity class, and (iii) fast approximate attention techniques ensure provable efficiency with only $1/\poly(n)$ error in $O(n^{2+o(1)})$ time. These findings confirm that high‐order flow matching is both theoretically sound and computationally tractable. In future work, we plan to empirically validate Second-Order $\MeanFlow$ on large‐scale image and video benchmarks, explore extensions such as discrete flow matching and adaptive timestepping. Our results open a promising direction for the development of richer, faster generative models grounded in high‐order differential dynamics. 


\newpage
\onecolumn
\appendix

\begin{center}
    \textbf{\LARGE Appendix }
\end{center}


{\bf Roadmap.}
In Section~\ref{sec:related}, we introduce work that related to our research.
In Section~\ref{sec:append:preli}, we present some backgrounds of this paper.
In Section~\ref{sec:append:proof_meanflow}, we supplement the missing proofs in Section~\ref{sec:feasibility_meanflow}.
In Section~\ref{sec:append:proof_circuit_complexity}, we provide the missing proofs in Section~\ref{sec:circuit_complexity}.
In Section~\ref{sec:append:proof_efficiency}, we show the missing proofs in Section~\ref{sec:provably_efficiency}.

\section{Related Work} \label{sec:related}

{\bf Flow Matching.} 
Flow Matching architecture~\cite{lcb+23,av23,lgl23} has emerged as a powerful alternative to simulation-based Continuous Normalizing Flows (CNFs)~\cite{crbd18}, providing a simulation-free objective by learning to regress velocity fields along pre-specified probability paths from simple noise distributions to complex data distributions.
Many works have already validated its capability on complicated real-world applications, e.g. image generation~\cite{lcb+23,ekb+24,fhla24,hglq24} and video generation~\cite{aaa+24,jsl+25,csy25}.
\cite{tfm+24} introduces Conditional Flow Matching (CFM) and its variant Optimal Transport Conditional Flow Matching (OT-CFM). CFM defines a class of simulation-free training objectives for CNFs, enabling efficient conditional generative modeling while accelerating both training and inference. OT-CFM further enhances this by approximating dynamic optimal transport without simulation, resulting in more stable and efficient learning.
Building on these foundations, recent work has proposed several innovative extensions to flow matching. For instance, \cite{kkn23} introduces an equivariant extension of flow matching for physics-based generative tasks. \cite{hppa24} presents Wasserstein Flow Matching, which generalizes traditional flow matching to different families of distributions, expanding its utility in areas such as computer graphics and genomics. \cite{grs+24} extends flow matching to discrete cases, which enables flow matching for discrete modeling, such as natural language. \cite{ccl+25} explores the integration of special relativity constraints into the flow matching framework. More recent advances on flow matching include flow matching on different geometry structures~\cite{clpl24,cl23,sgx+23,kpr+24}, PDEs~\cite{bgl+25,fbg+24,lzf+25,chm+24}, guided flow matching~\cite{zls+23,has+25,csy25}, benchmarking of generative models~\cite{cgs+25,ghh+25,ghs+25_physical,ghs+25_text}, and computational limits~\cite{lsy25,cll+25_var,ccl+25_form}.

{\bf Circuit Complexity.} 
As a cornerstone of theoretical computer science, circuit complexity investigates the computational capabilities and limitations of Boolean circuit families. This framework has recently found increasing application in the analysis of machine learning models, providing a principled approach to characterize their inherent computational capabilities. Several circuit complexity classes are particularly relevant in this domain. Crucial to this study is the class $\mathsf{AC}^0$, which consists of decision problems solvable by constant-depth Boolean circuits with unbounded fan-in and the logic gates AND, OR, and NOT, capturing problems that are highly parallelizable using standard logic gates. $\mathsf{TC}^0$ extends $\mathsf{AC}^0$ by incorporating MAJORITY gates (also known as threshold gates). $\mathsf{NC}^1$ represents languages recognizable by circuits of $O(\log n)$ depth with bounded gate arity~\cite{mss22}. A well-established inclusion chain in this context is $\mathsf{AC}^0 \subset \mathsf{TC}^0 \subseteq \mathsf{NC}^1$, but whether $\mathsf{TC}^0 = \mathsf{NC}^1$ holds an open question \cite{v99,ab09}.
Recent breakthroughs in circuit complexity have offered a rigorous framework for evaluating the computational limits of various neural network architectures. In particular, attention has turned to Transformer models and two of their key derivatives: SoftMax-Attention Transformers (SMATs) and Average-Head Attention Transformers (AHATs). Prior work has shown that SMATs have been demonstrated to admit efficient simulations via $L$-uniform circuits with comparable depth and gate constraints~\cite{lag+22}.  Similarly, AHATs can be efficiently simulated by non-uniform threshold circuits of constant depth, placing them within the $\mathsf{TC}^0$ class~\cite{mss22}. These results have been further strengthened by subsequent studies establishing that both AHAT and SMAT architectures can be approximated using $\mathsf{DLOGTIME}$-uniform $\mathsf{TC}^0$ circuits~\cite{ms23}, thereby reinforcing the connection between Transformer models and low-depth threshold circuit classes. Circuit complexity has been proven to be an effective tool for showing the fundamental limitations of deep architectures. Its success has recently been extended to $\mathsf{RoPE}$-Transformers~\cite{cll+25_rope,ysw+25,cssz25}, State Space Models~\cite{mps24,cll+25_mamba}, and Graph Neural Networks (GNNs)~\cite{bkm+20,lls+25}.
\section{Preliminary}\label{sec:append:preli}

This section provides the foundational background for the paper. In Section~\ref{sec:append:preli:flow_matching}, we introduce the principles of flow matching. In Section~\ref{sec:append:preli:circuit_complexity}, we present the basics of circuit complexity. In Section~\ref{sec:append:preli:vit}, we show the definition of the Vision Transformer formulation.

\subsection{Flow Matching}\label{sec:append:preli:flow_matching}
This section introduces the flow matching, covering both the training phase and sampling phase.
{\bf Training.} After formulating the basic concepts in flow matching, we define its training loss function. 
\begin{definition}[First-order flow matching loss, implicit in page 3 on~\cite{lcb+23}]\label{def:first_order_flow_matching_loss_function}

Let $t \in [0,1]$ denote the timestep. The trajectory $z_t \in \R^d$ is defined in Definition~\ref{def:trajectory} and the marginal velocity $v(z_t,t) \in \R^d$ is defined in Definition~\ref{def:marginal_velocity}. To learn the marginal velocity fields, a neural network $u_\theta$ is trained to minimize the following loss function:
    \begin{align*}
        L_{\mathrm{FO}}(\theta) := \E_{t, z_t} [\| u_{\theta}(z_t,t) - v(z_t,t) \|^2_2],
    \end{align*}
    where $t \sim {\sf Uniform}[0,1]$ and $z_t \sim p_t(z_t)$.
\end{definition}

Directly computing the flow matching loss is challenging due to the marginalization in Definition~\ref{def:marginal_velocity}. To overcome this, \cite{lcb+23} proposes the conditional flow matching loss as an alternative. 

\begin{definition}[Conditional first-order flow matching loss, implicit in page 4 on~\cite{lcb+23}]\label{def:conditional_first_order_flow_matching_loss_function}
Let $t \in [0,1]$ denote the timestep. The conditional velocity $v_t \in \R^d$ is defined in Definition~\ref{def:conditional_velocity}. The neural network $u_\theta$ is trained to minimize the following loss function:
    \begin{align*}
        L_{\mathrm{CFO}}(\theta) := \E_{t, x,\epsilon} [\| u_\theta(z_t,t) - v_t\|_2^2],
    \end{align*}
    where $t \sim {\sf Uniform}[0,1],x\sim \pdata(x)$, and $\epsilon \sim {\cal N}(\mu,\sigma^2 I_d)$.
\end{definition}

The effectiveness of the conditional flow matching loss is ensured by the following theorem. 
\begin{theorem}[Equivalence of loss functions, Theorem 2 on page 4 of~\cite{lcb+23}]\label{thm:equivalence_fm_loss_first_order}
    Up to a constant independent of $\theta$, the loss functions $L_{\mathrm{FO}}(\theta)$ in Definition~\ref{def:first_order_flow_matching_loss_function} and $L_{\mathrm{CFO}}(\theta)$ in Definition~\ref{def:conditional_first_order_flow_matching_loss_function} are equivalent, i.e., $\nabla_{\theta}L_{\mathrm{FO}}(\theta) = \nabla_{\theta}L_{\mathrm{CFO}}(\theta)$. Here $\mathrm{FO}$ denotes first-order, and $\mathrm{CFO}$ denotes conditional first-order.
\end{theorem}

{\bf Sampling. } 
After training the flow matching model, we can sample random noise from the prior distribution and move it to the real data distribution with the learned velocity fields. This process can be formulated by the ODE and the corresponding initial value problem ($\IVP$).

\begin{definition}[First-order ODE, implicit in page 2 on~\cite{lcb+23}]\label{def:ode}
Let $t \in [0,1]$ denote the timestep. The trajectory $z_t \in \R^d$ is defined in Definition~\ref{def:trajectory}, and the marginal velocity $v(z_t,t) \in \R^d$ is defined in Definition~\ref{def:marginal_velocity}. Samples can be generated by solving the following ordinary differential equation (ODE) for $z_t$:
\begin{align*}
    \frac{\d z_t}{\d t} := v(z_t,t).
\end{align*}

\end{definition}

\begin{definition}[First-order $\IVP$]\label{def:first_order_ivp}
Let $r,t \in [0,1]$ denote the two timesteps such that $r \leq t$. The trajectory $z_t \in \R^d$ is defined in Definition~\ref{def:trajectory}, and the marginal velocity $v(z_t,t) \in \R^d$ defined in Definition~\ref{def:marginal_velocity}. Then, we define the first-order Initial Value Problem $(\IVP)$ as follows: 
\begin{align*}
    z_r = z_t + \int_t^r v(z_\tau,\tau) \d \tau.
\end{align*}
\end{definition}

In practice, this integral in Definition~\ref{def:first_order_ivp} is numerically approximated over discrete time steps. To solve it with the Euler solver, we have the following fact:

\begin{fact}[Solving first-order $\MeanFlow$ $\IVP$ with $T$-step Euler solver]\label{fac:first_order_euler}
Given $T \in \mathbb{Z}_+$ as the total iteration of sampling, and arbitrary $\{t_i\}_{i=1}^T$ that satisfy $t_{i} > t_{i+1}$ for all $i \in \{1,\ldots,T-1\}$ as timestep scheduler. Then,
for all $i \in \{1,\ldots,T\}$, the first-order $\IVP$ defined as Definition~\ref{def:first_order_ivp} can be solved iteratively with the $T$-step Euler solver as:
\begin{align*}
    z_{t_{i}} = & ~ z_{t_{i-1}} + (t_{i} - t_{i-1})v(z_{t_{i-1}},t_i,t_{i-1}).
\end{align*}
\end{fact}

\begin{definition}[Second-order $\IVP$]\label{def:second_order_ivp}
Let $t,r \in [0,1]$ denote two different timesteps such that $r \le t$. The trajectory $z_t \in \R^d$, marginal velocity $v(z_t,t) \in \R^d$, and marginal acceleration $a(z_t,t) \in \R^d$ are defined in Definitions~\ref{def:trajectory}, \ref{def:marginal_velocity}, and \ref{def:marginal_acceleration}, respectively. Then, we define the second-order $\IVP$ as follows: 
    $
        z_r := z_t + \int_t^r(v(z_{\tau_2}, \tau_2) + \int^{\tau_2}_{t} a(z_{\tau_1},\tau_1) \d \tau_1)\d \tau_2.
    $
\end{definition}

\begin{remark}
    The first-order $\IVP$ in Definition~\ref{def:first_order_ivp} is a special case of the second-order $\IVP$ above, since we can let $a(z_{\tau},\tau) = 0$ for all $\tau \in [r,t]$ and exactly recover Definition~\ref{def:first_order_ivp}.
\end{remark}

\begin{fact}[Solving second-order $\MeanFlow$ $\IVP$ with with $T$-step Euler solver] \label{fac:second_order_euler}
Given $T \in \mathbb{Z}_+$ as the total iteration of sampling, and arbitrary $\{t_i\}_{i=1}^T$ that satisfy $t_{i} < t_{i+1}$ for all $i \in \{1,\ldots,T-1\}$ as timestep scheduler.
Then, for all $i \in \{0,1,\cdots,T\}$, the second-order $\IVP$ defined as Definition~\ref{def:second_order_ivp} can be solved iteratively with the Euler solver as:
$
    z_{t_i} =   z_{t_{i-1}} + (t_i - t_{i-1})v(z_{t_{i-1}},t_i,t_{i-1}) + \frac{1}{2}(t_{i} - t_{i-1})^2a(z_{t_{i-1}},t_i,t_{i-1}). 
$
\end{fact}

\subsection{Circuit Complexity Class}\label{sec:append:preli:circuit_complexity}
This section defines Boolean circuits and basic facts for the analysis of circuit complexity. We begin with the fundamental definition of a single Boolean circuit.

\begin{definition}[Boolean Circuit, implicit in page 102 on~\cite{ab09}]\label{def:boolean_circuit}
Let $n \in \Z_{+}$. A Boolean circuit is a directed acyclic graph (DAG) that computes a function $C_n : \{0,1\}^n \to \{0,1\}$. The nodes of the graph are called gates. Nodes with an in-degree of 0 are input nodes, representing the $n$ Boolean variables. All other gates compute a Boolean function of their inputs, which are the outputs of preceding gates.
\end{definition}
Since a single circuit only handles a fixed input length, we use a family of circuits to recognize languages containing strings of any length.
\begin{definition}[Circuit family recognizes languages, implicit in page 103 on~\cite{ab09}]\label{def:circuit_recognize_languages}
Let $L \subseteq \{0,1\}^*$ be a language and let $C = \{C_n\}_{n \in \mathbb{N}}$ be a family of Boolean circuits. We say that $C$ recognizes $L$ if for every string $x \in \{0,1\}^*$, the following holds:
    $C_{|x|}(x) = 1 \iff x \in L.$

\end{definition}
By placing constraints on the size and depth of these circuit families, we can define specific complexity classes, such as $\mathsf{NC}^i$.
\begin{definition}[$\mathsf{NC}^i$,implicit in page 40 on~\cite{ab09}]\label{def:nc}
The class $\mathsf{NC}^i$ consists of languages decidable by families of Boolean circuits with polynomial size $O(\mathrm{poly}(n))$, depth $O((\log n)^i)$, and composed of $\mathsf{AND}$, $\mathsf{OR}$, and $\mathsf{NOT}$ gates with bounded fan-in.
\end{definition}

Allowing $\mathsf{AND}$ and $\mathsf{OR}$ gates to have unbounded fan-in enables circuits to recognize a broader class of languages. This leads to the definition of the class $\mathsf{AC}^i$.

\begin{definition}[$\mathsf{AC}^i$,~\cite{ab09}]\label{def:ac}
The class $\mathsf{AC}^i$ is the set of languages recognizable by families of Boolean circuits with polynomial size $O(\mathrm{poly}(n))$ and depth $O((\log n)^i)$. These circuits use $\mathsf{NOT}$, $\mathsf{OR}$, and $\mathsf{AND}$ gates, where the $\mathsf{OR}$ and $\mathsf{AND}$ gates are permitted to have unbounded fan-in.
\end{definition}
Furthermore, since $\mathsf{NOT}$, $\mathsf{AND}$, and $\mathsf{OR}$ gates can be simulated by $\mathsf{MAJORITY}$ gates, we can define an even broader complexity class, $\mathsf{TC}^i$.

\begin{definition}[$\mathsf{TC}^i$,~\cite{ab09}]\label{def:tc}
$\mathsf{TC}^i$ consists of languages recognized by Boolean circuits with depth $O((\log n)^i)$, size $O(\mathrm{poly}(n))$, and unbounded fan-in gates for $\mathsf{NOT}$, $\mathsf{OR}$, $\mathsf{AND}$, and $\mathsf{MAJORITY}$, where a $\mathsf{MAJORITY}$ gate outputs $1$ when a majority of its inputs are active ($1$).
\end{definition}

\begin{remark}
The $\mathsf{MAJORITY}$ gates in the Definition~\ref{def:tc} of $\mathsf{TC}^i$ can be replaced by either $\mathsf{MOD}$ gates or $\mathsf{THRESHOLD}$ gates. A circuit containing any of these gate types is known as a threshold circuit.
\end{remark}

\begin{definition}[$\mathsf{P}$, implicit in page 27 on~\cite{ab09}]
The complexity class $\mathsf{P}$ is the set of all languages that can be decided by a deterministic Turing machine in polynomial time.
\end{definition}

\begin{fact}[Hierarchy folklore,~\cite{ab09,v99}]
The following inclusion relationships hold for all non-negative integers $i$:
    $\mathsf{NC}^i \subseteq \mathsf{AC}^i \subseteq \mathsf{TC}^i \subseteq \mathsf{NC}^{i+1} \subseteq \mathsf{P}.$
\end{fact}

\begin{definition}[$L$-uniform,~\cite{ab09}]\label{def:l_uniformity}
 A circuit family $C = \{C_n\}_{n \in \mathbb{N}}$ is L-uniform if a Turing machine exists that, on input $1^n$, generates a description of the circuit $C_n$ using $O(\log n)$ space. A language $L$ is in a class such as $L$-uniform $\mathsf{NC}^i$ if it is decided by an $L$-uniform circuit family $\{C_n\}$ that also meets the size and depth conditions of $\mathsf{NC}^i$.
\end{definition}

Next, we define a stricter form of uniformity based on a time constraint.

\begin{definition}[$\mathsf{DLOGTIME}$-uniform]\label{def:dlogtime_uniform}
A circuit family $C = \{C_n\}_{n \in \mathbb{N}}$ is $\mathsf{DLOGTIME}$-uniform if a Turing machine exists that, on input $1^n$, generates a description of the circuit $C_n$ in $O(\log n)$ time. A language is in a $\mathsf{DLOGTIME}$-uniform class if it is decided by a $\mathsf{DLOGTIME}$-uniform circuit family satisfying the corresponding resource constraints.
\end{definition}

We demonstrate several lemmas that define the basic operations' depth and width, which play fundamental roles in our circuit complexity analysis. First, we show that basic floating point operations can be implemented in $\mathsf{TC}^0$.
\begin{lemma}[Operations on floating point numbers in $\mathsf{TC}^0$, Lemma 10  and Lemma 11 of~\cite{chi24}]\label{lem:float_operations_TC} 
Assume the precision $p \leq \poly(n)$. Then we have:
\begin{itemize}
    \item Part 1. Given two $p$-bits float point numbers $x_1$ and $x_2$. Let the addition, division, and multiplication operations of $x_1$ and $x_2$ be outlined in \cite{chi24}. Then, these operations can be simulated by a size bounded by $\poly(n)$ and constant depth bounded by $d_{\rm std}$ $\mathsf{DLOGTIME}$-uniform threshold circuit.
    \item Part 2. Given $n$ $p$-bits float point number $x_1,\dots,x_n$. The iterated multiplication of $x_1, x_2\dots, x_n$ can be simulated by a size bounded by $\poly(n)$ and constant depth bounded by $d_\otimes$ $\mathsf{DLOGTIME}$-uniform threshold circuit.
    \item Part 3. Given $n$ $p$-bits float point number $x_1,\dots,x_n$. The iterated addition of $x_1, x_2\dots, x_n$ can be simulated by a size bounded by $\poly(n)$ and constant depth bounded by $d_\oplus$ $\mathsf{DLOGTIME}$-uniform threshold circuit. To be noticed, there is a rounding operation after the the summation is completed.
\end{itemize}
\end{lemma}

Next, we show that the exponential function can also be approximated in $\mathsf{TC}^0$.

\begin{lemma}[Approximating the Exponential Operation in $\mathsf{TC}^0$, Lemma 12  of~\cite{chi24}]\label{lem:exp}
    Assume the precision $p \leq \poly(n)$. Given any number $x$ with $p$-bit float point, the $\exp(x)$ function can be approximated by a uniform threshold circuit. This circuit has a size bounded by $\poly(n)$ and a constant depth $d_{\rm exp}$, and it guarantees a relative error of at most $2^{-p}$. 
\end{lemma}

Last, we demonstrate that the square root function can be approximated in $\mathsf{TC}^0$.
\begin{lemma}[Approximating the Square Root Operation in $\mathsf{TC}^0$, Lemma 12  of~\cite{chi24}]\label{lem:sqrt}
     Assume the precision $p \leq \poly(n)$. Given any number $x$ with $p$-bit float point, the $\sqrt{x}$ function can be approximated by a uniform threshold circuit. This circuit has a size bounded by $\poly(n)$ and a constant depth $d_\mathrm{sqrt}$, and it guarantees a relative error of at most $2^{-p}$. 
\end{lemma}

\subsection{ViT Model Formation}\label{sec:append:preli:vit}
In order to analyze the circuit complexity of $\MeanFlow$, we first provide the formal formulation for Vision Transformer (ViT) \cite{dbk+20}, which is an essential component in real-world $\MeanFlow$ implementation. First, we show the definition of the attention matrix.

\begin{definition}[Attention matrix]\label{def:attn_matrix}
    We use $W_Q, W_K \in \mathbb{F}_p^{d \times d}$ to denote the weight matrix of the query and key. Let $X \in \mathbb{F}_{p}^{n \times d}$ represent the input of the attention layer. Then, we use $A \in \mathbb{F}_{p}^{n \times n}$ to denote the attention matrix. Specifically, we denote the element of the attention matrix as the following:
    $A_{i,j} := \exp(  X_{i,*}   W_Q   W_K^\top   X_{j,*}^\top ).$
\end{definition}

Next, we demonstrate the definition of the single attention layer.
\begin{definition}[Single attention layer]\label{def:single_layer_transformer}
     We use $W_V \in \mathbb{F}_p^{d \times d}$ to denote the weight matrix of value. Let $X \in \mathbb{F}_{p}^{n \times d}$ represent the input of the attention layer. Let $A$ denote the attention matrix defined in Definition~\ref{def:attn_matrix}. Let $D := \diag( A {\bf 1}_n) $ denote a size $n \times n$ matrix. Then, we use $\mathsf{Attn}$ to denote the attention layer. Specifically, we have
   $ \mathsf{Attn} (X) :=  D^{-1} A X W_V .$
\end{definition}

Similarly, we define the multi-layer perceptron layer.
\begin{definition}[Multi-layer perceptron (MLP) layer]\label{def:mlp}
    Given an input matrix $X \in \mathbb{F}_p^{n \times d}$. Let $i \in [n]$. We use $g^{\rm MLP}$ to denote the MLP layer. Specifically, we have 
        $g^{\mathrm{MLP}}(X)_{i,*} := W \cdot X_{i,*} + b.$
\end{definition}

Furthermore, the definition of the layer-wise normalization layer is as follows.
\begin{definition}[Layer-wise normalization (LN) layer]\label{def:layer_norm}
     Given an input matrix $X \in \mathbb{F}_p^{n \times d}$. Let $i\in [n]$. We use $g^{\rm LN}$ to denote the LN layer. Specifically, we have 
        $g^{\mathrm{LN}} (X)_{i,*}  :=  \frac{X_{i,*} - \mu_i}{\sqrt{\sigma_i^2}},$
    where $\mu_i := \sum_{j=1}^d X_{i,j} / d$, and $\sigma_i^2 := \sum_{j = 1}^d (X_{i,j} - \mu_i)^2 / d$.
\end{definition}
By combining these definitions, we can define the Vision Transformer.
\begin{definition}[Vision Transformer (ViT)]\label{def:vit_with_residual}
    Assume the ViT has $m$ transformer layers. At the $i$-th layer, let $\mathsf{Attn}_i$ denote the self-attention layer (Definition~\ref{def:single_layer_transformer}), $\mathsf{LN}_{i,1}, \mathsf{LN}_{i,2}$ denote the two layer normalization layers (Definition~\ref{def:layer_norm}), and $\mathsf{MLP}_i$ denote the MLP layer  (Definition~\ref{def:mlp}). 

    Let $X_0 \in \mathbb{F}_p^{n \times d}$ be the input. Then, for each $i \in [m]$, define:
        $Y_i := \mathsf{Attn}_i \circ \mathsf{LN}_{i,1}(X_{i-1}) + X_{i-1} \in \mathbb{F}_p^{n \times d}, 
        X_i :=  \mathsf{MLP}_i \circ \mathsf{LN}_{i,2}(Y_i) + Y_i \in \mathbb{F}_p^{n \times d},$
    where $\circ$ denotes function composition.
    We use
        ${\sf ViT}(X_0) := X_m$
    to denote the computation of ViT as a function.
\end{definition}

\section{Feasibility of Second-Order \texorpdfstring{$\MeanFlow$}{MeanFlow}}\label{sec:append:proof_meanflow}


This section supplies the proofs omitted from Section~\ref{sec:feasibility_meanflow}. In Section~\ref{sec:append:proof_meanflow:boundary_conditions}, we prove the boundary conditions for both average velocity and average acceleration. In Section~\ref{sec:append:proof_meanflow:consistency}, we prove the consistency constraints for both average velocity and average acceleration. In Section~\ref{sec:append:proof_meanflow:equivalence}, we prove the equivalence of the $\MeanFlow$ loss functions. In Section~\ref{sec:append:proof_meanflow:efficiency}, we demonstrate the effectiveness of the $\MeanFlow$ and Second-order $\MeanFlow$ loss function.

\subsection{Boundary Conditions}\label{sec:append:proof_meanflow:boundary_conditions}
In this section, we establish the boundary conditions of average velocity and acceleration. We begin with the boundary condition for average velocity, which connects it to the marginal velocity.

\begin{lemma}[Boundary condition of average velocity, formal version of Lemma~\ref{lem:boundary_conditions_avg_v}]\label{lem:boundary_conditions_avg_v:formal}

The relationship between the average velocity $\bar{v}$ (Definition~\ref{def:average_velocity}) and the marginal velocity $v$ (Definition~\ref{def:marginal_velocity}) is established by the boundary condition in Definition~\ref{def:gen_consis_func} as follows:
\begin{align*}
    \ov{v}(z_r,r,r) = v(z_r,r),
\end{align*}
and
\begin{align*}
    \ov{v}(z_t,t,t) = v(z_t,t).
\end{align*}
\end{lemma}
\begin{proof}
{\bf Part 1: $\ov{v}(z_r,r,r) = v(z_r,r)$.}
By Definition~\ref{def:average_velocity}, the average velocity is given by:
\begin{align*}
    \ov{v}(z_t, r, t) = \frac{1}{t-r} \int_{r}^{t} v(z_{\tau}, \tau) d\tau.
\end{align*}

We can view this as the mean value of the function $v$ over the interval $[r, t]$. As the interval length $t-r$ approaches $0$, this mean value converges to the value of the integrand at that point. Thus, we have:
\begin{align*}
    \lim_{t \to r} \ov{v}(z_t, r, t) 
    = & ~ \lim_{t \to r} \frac{1}{t-r} \int_{r}^{t} v(z_{\tau}, \tau) d\tau \\ \notag
    = & ~ v(z_r, r),
\end{align*}
where the first step follows from the definition of $\ov{v}$ in Definition~\ref{def:average_velocity}, and the second step follows from the fundamental theorem of calculus.

{\bf Part 2: $\ov{v}(z_t,t,t) = v(z_t,t)$.}
According to Definition~\ref{def:average_velocity}, the average velocity $\ov{v}$ is defined as the mean value of the velocity function v over the interval $[r,t]$:
\begin{align*}
    \ov{v}(z_t, r, t) = \frac{1}{t-r} \int_{r}^{t} v(z_{\tau}, \tau) d\tau.
\end{align*}
As the interval length $(t - r)$ shrinks to zero, this mean value converges to the value of the integrand at that point. We can show this formally by taking the limit as $r \rightarrow t$. By the Fundamental Theorem of Calculus, we have:
\begin{align*}
\lim_{r \to t} \ov{v}(z_t, r, t)
&= \lim_{r \to t} \frac{1}{t-r} \int_{r}^{t} v(z_{\tau}, \tau) d\tau \\
&= v(z_t, t).
\end{align*}
The first step applies the definition of $\ov{v}$, and the second is a direct result of the theorem. This demonstrates that the average velocity $\ov{v}$ converges to the instantaneous (or marginal) velocity v in the limit.

Thus, we prove that the average velocity $\ov{v}$ equals the marginal velocity $v$ at the limit.
\end{proof}

We now apply the same logic to acceleration. The following formalizes how the average acceleration converges to the marginal acceleration.

\begin{lemma}[Boundary conditions of average acceleration, formal version of Lemma~\ref{lem:boundary_conditions_avg_a}]\label{lem:boundary_conditions_avg_a:formal}
Given the average acceleration $\ov{a}$ defined in Definition~\ref{def:average_acceleration} and the marginal acceleration $a$ defined in Definition~\ref{def:marginal_acceleration}, $\ov{a}$ converges to $a$ as the time interval shrinks to zero. Specifically:
\begin{align*}
    \ov{a}(z_r,r,r) = a(z_r,r)
\end{align*}
and
\begin{align*}
    \ov{a}(z_t,t,t) = a(z_t,t).
\end{align*}
\end{lemma}

\begin{proof}
{\bf Part 1: $\ov{a}(z_t,t,t) = a(z_t,t)$.}By Definition~\ref{def:average_acceleration}, the average acceleration over the interval $[r, t]$ is:
\begin{align*}
    \ov{a}(z_t, r, t) = \frac{1}{t-r} \int_{r}^{t} a(z_{\tau}, \tau) d\tau.
\end{align*}
Taking the limit as $r \to t$, we get:
\begin{align*}
    \lim_{r \to t} \ov{a}(z_t, r, t)
    & = \lim_{r \to t} \frac{\int_{r}^{t} a(z_{\tau}, \tau) d\tau}{t-r} \\
    & = a(z_t, t),
\end{align*}
where the first step follows from the definition of $\ov{a}$ in Definition~\ref{def:average_acceleration}, and the second step follows from the fundamental theorem of calculus.

{\bf Part 2: $\ov{a}(z_r,r,r) = a(z_r,r)$.} The proof is symmetrical to Part 1. We take the limit of the average acceleration over the interval $[r,t]$ as $t$ approaches $r$. By the same application of the Fundamental Theorem of Calculus, this limit resolves to the instantaneous acceleration at time $r$:
\begin{align*}
    \lim_{t \to r} \ov{a}(z_t, r, t)
    & = \lim_{t \to r} \frac{\int_{r}^{t} a(z_{\tau}, \tau) d\tau}{t-r} \\
    & = a(z_r, r).
\end{align*}


Thus, we complete the proof.
\end{proof}

\subsection{Consistency constraint}\label{sec:append:proof_meanflow:consistency}
In this section, we give the consistency constraints for average velocity and acceleration. We begin by demonstrating that average velocity (Definition~\ref{def:average_velocity}) satisfies the consistency constraint.
\begin{lemma}[Consistency constraint of average velocity, formal version of Lemma~\ref{lem:consistency_constraint_avg_v}]\label{lem:consistency_constraint_avg_v:formal}
Let $r,t$ denote two different timesteps. The timestep $s$ is any intermediate time between $r$ and $t$. The marginal velocity $v$ is defined in Definition~\ref{def:marginal_velocity}. The average velocity $\ov{v}$ defined in Definition~\ref{def:average_velocity} satisfies an additive consistency constraint:
\begin{align*}
    (t-r)\ov{v}(z_t, r, t) = (s-r)\ov{v}(z_s, r, s) + (t-s)\ov{v}(z_t, s, t).
\end{align*}
\end{lemma}
\begin{proof}
For the marginal velocity $v$ integrated over time $\tau$, we have:
\begin{align*}
    \int_{r}^{t} v(z_{\tau}, \tau) d\tau = \int_{r}^{s} v(z_{\tau}, \tau) d\tau + \int_{s}^{t} v(z_{\tau}, \tau) d\tau,
\end{align*}
where the first step follows from the additivity of the integral.

From the definition of average velocity in Definition~\ref{def:average_velocity}, $\ov{v}(z_t, r, t) = \frac{1}{t-r} \int_{r}^{t} v d\tau$, we can express the displacement (the integral) as $(t-r)\ov{v}(z_t, r, t)$. By substituting this relationship into the integral additivity equation above, we arrive at the consistency constraint:
\begin{align*}
    (t-r)\ov{v}(z_t, r, t) = (s-r)\ov{v}(z_s, r, s) + (t-s)\ov{v}(z_t, s, t).
\end{align*}

This demonstrates that $\ov{v}$ satisfies its inherent consistency constraint. 
    
\end{proof}

Next, we show that the consistency constraint applies to average acceleration (Definition~\ref{def:average_acceleration}).

\begin{lemma}[Consistency constraint of average acceleration, formal version of Lemma~\ref{lem:consistency_constraint_avg_a}] \label{lem:consistency_constraint_avg_a:formal}
Let $r, s,$ and $t$ be three distinct time steps, such that $s$ lies between $r$ and $t$. Given the marginal acceleration $a$ (as defined in Definition~\ref{def:marginal_acceleration}), the average acceleration $\ov{a}$ (as defined in Definition~\ref{def:average_acceleration}) satisfies the following additive consistency constraint:
\begin{align*}
    (t-r)\ov{a}(z_t, r, t) = (s-r)\ov{a}(z_s, r, s) + (t-s)\ov{a}(z_t, s, t).
\end{align*}
\end{lemma}

\begin{proof}
The integral of marginal acceleration $a(z_{\tau}, \tau)$ over time $\tau$ exhibits additivity:
\begin{align*}
\int_{r}^{t} a(z_{\tau}, \tau) d\tau = \int_{r}^{s} a(z_{\tau}, \tau) d\tau + \int_{s}^{t} a(z_{\tau}, \tau) d\tau.
\end{align*}

From Definition~\ref{def:average_acceleration}, the average acceleration $\ov{a}(z_t, r, t)$ is defined as $\frac{1}{t-r} \int_{r}^{t} a(z_{\tau}, \tau) d\tau$. This allows us to express the integral of acceleration over a given time interval as the product of the time interval and the average acceleration over that interval. Specifically, $\int_{r}^{t} a(z_{\tau}, \tau) d\tau = (t-r)\ov{a}(z_t, r, t)$.

Substituting this relationship into the additive integral equation yields the consistency constraint:
\begin{align*}
(t-r)\ov{a}(z_t, r, t) = (s-r)\ov{a}(z_s, r, s) + (t-s)\ov{a}(z_t, s, t).
\end{align*}

Thus, we complete the proof.
\end{proof}

\subsection{Equivalence of \texorpdfstring{$\MeanFlow$}{MeanFlow} loss functions}\label{sec:append:proof_meanflow:equivalence}
In this section, we show the equivalence between the practial and ideal loss functions for first-order $\MeanFlow$.
\begin{theorem}[Equivalence of $\MeanFlow$ loss functions, formal version of Theorem~\ref{thm:equi_meanflow_loss}]\label{thm:equi_meanflow_loss:formal}
    The loss function $L_{\mathrm{FOM}}(\theta)$ in Definition~\ref{def:first_order_meanflow_loss} is equivalent to the ideal loss function $L_{\mathrm{FOM}}^*(\theta)$ in Definition~\ref{def:ideal_first_order_meanflow_loss}. Here $\mathrm{FOM}$ denotes first-order $\MeanFlow$.

\end{theorem}
\begin{proof}
    We prove the equivalence of Eq.~\eqref{eq:average_velocity_int} and Eq.~\eqref{eq:average_velocity_diff} from both sides.

{\bf Part 1: Eq.~\eqref{eq:average_velocity_int} $\implies$ Eq.~\eqref{eq:average_velocity_diff}.}
The definition of Eq.~\eqref{eq:average_velocity_int} is as follows:
\begin{align*}
    \ov{v}(z_t,r,t) = \frac{1}{t-r} \int_r^t v(z_\tau,\tau) \d \tau.
\end{align*}

We rewrite Eq.~\eqref{eq:average_velocity_int} as:
\begin{align}\label{eq:average_velocity_rewrite}
     (t-r)\ov{v}(z_t,r,t) = \int_r^t v(z_\tau,\tau) \d \tau.
\end{align}

Next, we differentiate both sides of Eq.~\eqref{eq:average_velocity_rewrite} with respect to $t$, treating $r$ as independent of $t$. This yields:
\begin{align*}
     \frac{\d}{\d t} (t-r)\ov{v}(z_t,r,t) = \frac{\d}{\d t} \int_r^t v(z_\tau,\tau) \d \tau.
\end{align*}

Applying the product rule to the left-hand side and the fundamental theorem of calculus to the right-hand side, we obtain:
\begin{align*}
     \ov{v}(z_t,r,t) + (t-r) \frac{\d}{\d t} \ov{v}(z_t,r,t) = v(z_\tau,\tau),
\end{align*}

Rearranging the terms, we obtaion Eq.~\eqref{eq:average_velocity_diff}:
\begin{align*}
     \ov{v}(z_t,r,t) = v(z_t,t)  - (t-r) \frac{\d}{\d t} \ov{v}(z_t,r,t).
\end{align*}



{\bf Part 2: Eq.~\eqref{eq:average_velocity_diff} $\implies $ Eq.~\eqref{eq:average_velocity_int}.}
The definition of Eq.~\eqref{eq:average_velocity_diff} is as follows:
\begin{align*}
    \ov{v}(z_t,r,t) = v(z_t,t)  - (t-r) \frac{\d}{\d t} \ov{v}(z_t,r,t).
\end{align*}

We rewrite Eq.~\eqref{eq:average_velocity_diff} as follows:
\begin{align}\label{eq:marginal_average_velocity}
    v(z_t,t) 
    = & ~ \ov{v}(z_t,r,t) + (t-r) \frac{\d}{\d t} \ov{v}(z_t,r,t) \\ \notag
    = & ~ \ov{v}(z_t,r,t) \frac{\d}{\d t}(t-r) + (t-r) \frac{\d}{\d t} \ov{v}(z_t,r,t) \\ \notag
    = & ~ \frac{\d}{\d t}(t-r)\ov{v}(z_t,r,t),
\end{align}
where the first step follows from the definition of Eq.~\eqref{eq:average_velocity_diff}, the second step follows from the product rule of differentiation, and the third step follows from $\frac{\d}{\d t}(t-r) = 1$.

In general, equality of derivatives does not imply equality of integrals: they may differ by a constant. Therefore, based on the third step of Equation \eqref{eq:marginal_average_velocity}, the following equation can be directly implied:
\begin{align}\label{eq:average_velocity_int_c1_c2}
    (t-r)\ov{v}(z_t,r,t) + C_1 = \int_r^t v(z_\tau,\tau)d\tau + C_2.
\end{align}

Choosing $t=r$ in Eq.~\eqref{eq:average_velocity_int_c1_c2}, we have $C_1=C_2$. Substituting $C_1=C_2$ into Eq~\eqref{eq:average_velocity_int_c1_c2}, we get the following equation:
\begin{align*}
    (t-r)\ov{v}(z_t,r,t) = \int_r^t v(z_\tau,\tau)d\tau.
\end{align*}

We rewrite the equation to obtain Eq.~\eqref{eq:average_velocity_int} as follows:
\begin{align*}
    \ov{v}(z_t,r,t) = \frac{1}{(t-r)}\int_r^t v(z_\tau,\tau)d\tau.
\end{align*}



Thus, we complete the proof.

\end{proof}

\subsection{Effectiveness of \texorpdfstring{$\MeanFlow$}{MeanFlow} and Second-order \texorpdfstring{$\MeanFlow$}{MeanFlow} loss function }\label{sec:append:proof_meanflow:efficiency}
In this section, we present the effectiveness of $\MeanFlow$ and second-order $\MeanFlow$ loss function. First, we show that the first-order $\MeanFlow$ loss function can be computed efficiently.

\begin{theorem}
[Efficiency of $\MeanFlow$ loss, formal version of Theorem~\ref{thm:eff_meanflow_loss}]\label{thm:eff_meanflow_loss:formal}
    The loss function $L_{\mathrm{FOM}}(\theta)$ in Definition~\ref{def:first_order_meanflow_loss} can be evaluated efficiently with the conditional velocity $v_t$ in Definition~\ref{def:conditional_velocity}  and Jacobian-Vector Products (JVPs) of the neural network $u_{1, \theta_1}$, i.e.,
    \begin{align*}
        \ov{v}(z_t,r,t) = v_t - (t-r)(v_t \partial_z u_\theta + \partial_t u_\theta).
    \end{align*}
\end{theorem}
\begin{proof}
    We fist compute the total derivative of $ \ov{v}(z_t,r,t)$ as follows:
    \begin{align}\label{eq:average_velocity_jvp}
        \frac{\d}{\d t} \ov{v}(z_t,r,t) 
        = & ~ \frac{\d z_t}{\d t}\partial_{z_t}u + \frac{\d r}{\d t}\partial_{r}u + \frac{\d t}{\d t}\partial_{t}u \\ \notag
        = & ~ v(z_t,t) \partial_{z_t}u + \partial_{t}u,
    \end{align}
    where the first step follows from the chain rule and the second step follows from $\frac{\d z_t}{\d t} = v(z_t,t)$ in Definition~\ref{def:ode}, $\frac{\d r}{\d t} = 0$ and $\frac{\d t}{\d t} = 1$.

    Besides, we also have 
    \begin{align}\label{eq:efficiency_meanflow_step2}
        \ov{v}(z_t,r,t) 
        = & ~ \frac{1}{t-r} \int_r^t v(z_\tau,\tau) \d \tau \notag \\
        = & ~  v(z_t,t)  - (t-r) \frac{\d}{\d t} \ov{v}(z_t,r,t) \notag \\
        = & ~ v(z_t,t)  - (t-r)(v(z_t,t) \partial_{z_t}\ov{v} + \partial_{t}\ov{v}),
    \end{align}
where the first step follows from definition of $\ov{v}(z_t,r,t)$ in Definition~\ref{def:average_velocity}, the second step follows from Theorem~\ref{thm:equi_meanflow_loss}, and the third step follows from Eq.~\eqref{eq:average_velocity_jvp}. 

The velocity $v(z_t,t)$ given in Eq.~\eqref{eq:efficiency_meanflow_step2} is the marginal velocity. Following~\cite{lcb+23}, we replace this with the conditional velocity. Then, the objective function is:
\begin{align*}
    \ov{v}(z_t,r,t) = v_t - (t-r)(v_t \partial_z u_\theta + \partial_t u_\theta).
\end{align*}

This finishes the proof. 
\end{proof}
Next, we establish the equivalence between the practical and ideal second-order loss functions.
\begin{theorem}[Effectivenss of second-order $\MeanFlow$, formal version of Theorem~\ref{thm:equi_2nd_meanflow_loss}]\label{thm:equi_2nd_meanflow_loss:formal}
    The loss function $L_{\mathrm{SOM}}(\theta)$ in Definition~\ref{def:second_order_meanflow_loss} is equivalent to the ideal loss function $L_{\mathrm{SOM}}^*(\theta)$ in Definition~\ref{def:ideal_second_order_meanflow_loss}. Here $\mathrm{SOM}$ denotes second-order $\MeanFlow$.
\end{theorem}
\begin{proof}
    The key to showing the effectiveness of the second-order $\MeanFlow$ is showing the equivalence between Eq.~\eqref{eq:average_acceleration_int} of Definition~\ref{def:average_acceleration} and the following equation:
\begin{align}\label{eq:average_acceleration_diff}
    \ov{a}(z_t,r,t) = a(z_t,t) - (t-r)\frac{\d}{\d t}\ov{a}(z_t,r,t).
\end{align}

We prove the equivalence of the integral form Eq.~\eqref{eq:average_acceleration_int} and the differential form Eq.~\eqref{eq:average_acceleration_diff} in two parts.

\paragraph{Part 1: Eq.~\eqref{eq:average_acceleration_int}  $\implies$ Eq.~\eqref{eq:average_acceleration_diff}}
The integral definition of the average acceleration is:
\begin{align*}
    \ov{a}(z_t,r,t) = \frac{1}{t-r} \int_r^t a(z_\tau,\tau) \d \tau.
\end{align*}
Multiplying by $(t-r)$, we get:
\begin{align}\label{eq:average_acceleration_rewrite}
     (t-r)\ov{a}(z_t,r,t) = \int_r^t a(z_\tau,\tau) \d \tau.
\end{align}
Next, we differentiate both sides of Eq.~\eqref{eq:average_acceleration_rewrite} with respect to $t$, treating $r$ as a constant:
\begin{align*}
     \frac{\d}{\d t} \left[(t-r)\ov{a}(z_t,r,t)\right] = \frac{\d}{\d t} \int_r^t a(z_\tau,\tau) \d \tau.
\end{align*}
Applying the product rule to the left-hand side and the Fundamental Theorem of Calculus to the right-hand side yields:
\begin{align*}
     \ov{a}(z_t,r,t) + (t-r) \frac{\d}{\d t} \ov{a}(z_t,r,t) = a(z_t,t).
\end{align*}
Rearranging the terms, we obtain the differential form in Eq.~\eqref{eq:average_acceleration_diff}:
\begin{align*}
     \ov{a}(z_t,r,t) = a(z_t,t) - (t-r) \frac{\d}{\d t} \ov{a}(z_t,r,t).
\end{align*}

\paragraph{Part 2: Eq.~\eqref{eq:average_acceleration_diff} $\implies$ Eq.~\eqref{eq:average_acceleration_int}}
We start with the differential form from Eq.~\eqref{eq:average_acceleration_diff}:
\begin{align*}
    \ov{a}(z_t,r,t) = a(z_t,t) - (t-r) \frac{\d}{\d t} \ov{a}(z_t,r,t).
\end{align*}
Rearranging the equation to isolate $a(z_t,t)$, we find:
\begin{align*}
    a(z_t,t)
    & = \ov{a}(z_t,r,t) + (t-r) \frac{\d}{\d t} \ov{a}(z_t,r,t).
\end{align*}
We recognize the right-hand side as the result of the product rule for differentiation applied to $(t-r)\ov{a}(z_t,r,t)$:
\begin{align*}
    a(z_t,t) = \frac{\d}{\d t}\left[(t-r)\ov{a}(z_t,r,t)\right].
\end{align*}
Now, we integrate both sides with respect to a dummy variable $\tau$ from $r$ to $t$:
\begin{align*}
    \int_r^t a(z_\tau,\tau) \d\tau = \int_r^t \frac{\d}{\d \tau}\left[(\tau-r)\ov{a}(z_\tau,r,\tau)\right] \d\tau.
\end{align*}
Applying the Fundamental Theorem of Calculus, we have:
\begin{align*}
    \int_r^t a(z_\tau,\tau) \d\tau = (t-r)\ov{a}(z_t,r,t)
\end{align*}
Finally, dividing by $(t-r)$ for $t \neq r$ yields the integral definition from Eq.~\eqref{eq:average_acceleration_int}:
\begin{align*}
    \ov{a}(z_t,r,t) = \frac{1}{t-r}\int_r^t a(z_\tau,\tau) \d\tau.
\end{align*}
Therefore, applying Theorem~\ref{thm:equivalence_fm_loss_first_order}, we finish the proof of equivalence.

\end{proof}

Finally, we demonstrate that the second-order loss function is also computed efficiently.
\begin{theorem}[Efficiency of second-order $\MeanFlow$, formal version of Theorem~\ref{thm:efficiency_2nd_meanflow_loss}]\label{thm:efficiency_2nd_meanflow_loss:formal}
    The second-order $\MeanFlow$ loss function $L_{\mathrm{SOM}}(\theta)$ in Definition~\ref{def:second_order_meanflow_loss} can be evaluated efficiently with the conditional velocity $v_t$ and conditional acceleration $a_t$ in and Jacobian-Vector Products (JVPs) of the neural networks $u_{1, \theta_1}$ and $u_{2, \theta_2}$, i.e.,
    \begin{align*}
        \ov{v}(z_t,r,t) = &  ~v_t - (t-r)(v_t \partial_z u_{1, \theta_1} + \partial_t u_{1, \theta_1}), \\
         \ov{a}(z_t,r,t) = & ~a_t - (t-r)(a_t \partial_z u_{2, \theta_2} + \partial_t u_{2, \theta_2}).
    \end{align*}
\end{theorem}
\begin{proof}
We first the compute the total derivative of $\ov{a}(z_t,r,t)$ as follows:
\begin{align}\label{eq:average_acceleration_jvp}
    \frac{\d}{\d t} \ov{a}(z_t,r,t) 
    = & ~ \frac{\d z_t}{\d t}\partial_{z_t}u + \frac{\d r}{\d t}\partial_{r}u + \frac{\d t}{\d t}\partial_{t}u \\ \notag
    = & ~ v(z_t,t) \partial_{z_t}u + \partial_{t}u,
\end{align}
where the first step is due to the chain rule and the second step is obtained by $\frac{\d z_t}{\d t} = v(z_t,t)$ in Definition~\ref{def:ode}, $\frac{\d r}{\d t} = 0$ and $\frac{\d t}{\d t} = 1$.

Next, we have:
    \begin{align*}
        \ov{a}(z_t,r,t) 
        = & ~ \frac{1}{t-r} \int_r^t v(z_\tau,\tau) \d \tau \\
        = & ~  a(z_t,t)  - (t-r) \frac{\d}{\d t} \ov{a}(z_t,r,t) \\
        = & ~ a(z_t,t)  - (t-r)(v(z_t,t) \partial_{z_t}\ov{a} + \partial_{t}\ov{a}),
    \end{align*}
where the first step is due to definition of $\ov{a}(z_t,r,t)$ in Definition~\ref{def:average_acceleration}, the second step is obtained by Theorem~\ref{thm:equi_2nd_meanflow_loss}, and the third step follows from Eq.~\eqref{eq:average_acceleration_jvp}.

This finishes the proof.
\end{proof}
\section{Circuit Complexity of Second-order \texorpdfstring{$\MeanFlow$}{MeanFlow}}\label{sec:append:proof_circuit_complexity}

This section provides the missing proofs for  Section~\ref{sec:circuit_complexity}. In Section~\ref{sec:append:proof_circuit_complexity:viT_computation},we prove that ViT computation belongs to $\mathsf{TC}^0$. In Section~\ref{sec:append:proof_circuit_complexity:euler_solver}, we show that sampling with a T-step Euler solver for both $\MeanFlow$ and second-order $\MeanFlow$ also belongs to $\mathsf{TC}^0$.

\subsection{Circuit Complexity of ViT}\label{sec:append:proof_circuit_complexity:viT_computation}
In this section, we demonstrate the circuit complexity of ViT. First, we analyze the complexity of computing the value $Y_i$ in the self-attention block.
\begin{lemma}[$Y_i$ of ViT computation in $\mathsf{TC}^0$, formal version of Lemma~\ref{lem:vit_y_tc0}]\label{lem:vit_y_tc0:formal}
    Assume the precision $p \leq \poly(n)$, then for a $m$ layer ViT, we can use a size bounded by $\poly(n)$ and constant depth $m(9d_{\rm std} + 5d_{\oplus} + d_{\rm sqrt} + d_{\rm exp})$ uniform threshold circuit to simulate the computation of $Y_i$ defined in Definition~\ref{def:vit_with_residual}.
\end{lemma}
\begin{proof}
    By using the result of Lemma~\ref{lem:layer_tc0}, we can apply a size bounded by $\poly(n)$ and a constant depth $5d_{\rm std} + 2d_\oplus + d_{\rm sqrt}$ uniform threshold circuit to simulate the computation of $\mathsf{LN}_{i,1}$ layer, defined in Definition~\ref{def:layer_norm}, for each $i \in [m]$.

    By using the result of Lemma~\ref{lem:attn_tc0}, we can apply a size bounded by $\poly(n)$ and a constant depth $3(d_{\rm std} + d_\oplus) + d_{\rm exp}$ uniform threshold circuit to simulate $\mathsf{Attn}_i$ defined in Definition~\ref{def:single_layer_transformer}, for each $i \in [m]$
    
    For the residual addition, since the matrices have dimensions $n \times d$, we can use $nd \leq O(n^2)$ parallel adders. The overall circuit has depth $d_{\rm std}$ and size $\poly(n)$. Therefore, for each $i \in [m]$, we require a circuit of size $\poly(n)$ and depth $d_{\rm std}$.
    
    Then, we can have that the size of the uniform threshold circuit is bounded by $\poly(n)$, and the total depth of the circuit is $m(9d_{\rm std} + 5d_{\oplus} + d_{\rm sqrt} + d_{\rm exp})$.
\end{proof}

Next, we give the circuit complexity for computing the value $X_i$ in the MLP block.

\begin{lemma}[$X_i$ of ViT computation in $\mathsf{TC}^0$, formal version of Lemma~\ref{lem:vit_x_tc0}] \label{lem:vit_x_tc0:formal}
    Assume the precision $p \leq \poly(n)$, then for a $m$ layer ViT, we can use a size bounded by $\poly(n)$ and constant depth $m(8d_{\rm std} + 3d_{\oplus} + d_{\rm sqrt})$ uniform threshold circuit to simulate the computation of $X_i$ defined in Definition~\ref{def:vit_with_residual}.
\end{lemma}
\begin{proof}
    By using the result of Lemma~\ref{lem:layer_tc0}, we can apply a size bounded by $\poly(n)$ and a constant depth $5d_{\rm std} + 2d_\oplus + d_{\rm sqrt}$ uniform threshold circuit to simulate the computation of $\mathsf{LN}_{i,2}$ layer, defined in Definition~\ref{def:layer_norm}, for each $i \in [m]$.

    By using the result of Lemma~\ref{lem:mlp_tc0}, we can apply a size bounded by $\poly(n)$ and a constant depth $2d_{\rm std}+d_{\oplus}$ uniform threshold circuit to simulate $\mathsf{MLP}_i$ defined in Definition~\ref{def:mlp}, for each $i \in [m]$

    For the residual addition, since the matrices have dimensions $n \times d$, we can use $n(d+2)$ parallel adders. The overall circuit has depth $d_{\rm std}$ and size $\poly(n)$. Therefore, for each $i \in [m]$, we require a circuit with size $\poly(n)$ and depth $d_{\rm std}$.

    Then, we can have that the size of the uniform threshold circuit is bounded by $\poly(n)$, and the total depth of the circuit is $m(8d_{\rm std} + 3d_{\oplus} + d_{\rm sqrt})$.
\end{proof}
Finally, we combine these lemmas to show the ViT can be simulated by a $\mathsf{TC}^0$ circuit.
\begin{lemma}[ViT computation in $\mathsf{TC}^0$, formal version of Lemma~\ref{lem:vit_tc0}]\label{lem:vit_tc0:formal}
    Assume the number of transformer layers $m = O(1)$. Assume the precision $p \leq \poly(n)$. Subsequently, we can apply a uniform threshold circuit to simulate the ViT defined in Definition~\ref{def:vit_with_residual}. The circuit has size $\poly(n)$ and $O(1)$ depth.
\end{lemma}
\begin{proof}    
    Applying Lemma~\ref{lem:vit_y_tc0}, we can apply a size bounded by $\poly(n)$ and a constant depth $m(9d_{\rm std} + 5d_{\oplus} + d_{\rm sqrt} + d_{\rm exp})$ uniform threshold circuit to simulate the computation of $Y_i$ defined in Definition~\ref{def:vit_with_residual}.

    By using the result of Lemma~\ref{lem:vit_x_tc0}, we can apply a size bounded by $\poly(n)$ and a constant depth $m(8d_{\rm std} + 3d_{\oplus} + d_{\rm sqrt})$ uniform threshold circuit to simulate the computation of $X_i$ defined in Definition~\ref{def:vit_with_residual}.

    Thus, we can have that the size of the uniform threshold circuit is bounded by $\poly(n)$, and the total depth of the circuit is $m(17d_{\rm std} + 8d_{\oplus} + 2d_{\rm sqrt} + d_{\rm exp})$, since $m = O(1)$, thus we have $O(1)$ total depth.
\end{proof}

\subsection{Circuit Complexity of \texorpdfstring{$\MeanFlow$}{MeanFlow} with Euler
Solver}\label{sec:append:proof_circuit_complexity:euler_solver}
In this section, we show the circuit complexity of $\MeanFlow$ with Euler Solver. We begin by establishing the circuit complexity for the sampling process of the first-order $\MeanFlow$ model.
\begin{theorem}[$\MeanFlow$ sampling with $T$-step Euler solver belongs to $\mathsf{TC}^0$, formal version of Theorem~\ref{thm:meanflow_sampling_tc0}] \label{thm:meanflow_sampling_tc0:formal}
    Given $T \in \mathbb{Z}_+$ as the total iteration of sampling, and arbitrary $\{t_i\}_{i=1}^T$ that satisfy $t_{i} > t_{i+1}$ for all $i \in \{1,\ldots,T-1\}$ as timestep scheduler. 
    Assume the precision $p \leq \poly(n)$, the number of transformer layers $m = O(1)$, and $T = O(1)$.
    Then we can use a size bounded by $\poly(n)$ and constant depth $T(m(17d_{\rm std} + 8d_{\oplus} + 2d_{\rm sqrt} + d_{\rm exp}) + 5d_{\rm std})$ uniform threshold circuit to simulate the $\MeanFlow$ $T$-step sampling defined in Definition~\ref{def:meanflow_sampling_vit}.
\end{theorem}

\begin{proof}
    The operation $t_i {\bf 1}_n$ and $t_{i-1} {\bf 1}_n$ can be implemented by a uniform threshold circuit of size $n < \poly(n)$ and depth $\d_{\rm std}$.

    
    Using Lemma~\ref{lem:vit_tc0}, we can simulate the $\sf ViT$ (Definition~\ref{def:vit_with_residual}) with a uniform threshold circuit. This circuit has a size bounded by $\poly(n)$ and a constant depth $m(17d_{\rm std} + 8d_{\oplus} + 2d_{\rm sqrt} + d_{\rm exp})$.

    The $t_i - t_{i-1}$ operation can be simulated with a uniform threshold circuit with size of $1 < \poly(n)$ and depth of $d_{\rm std}$.

    The multiplication between result of $t_i - t_{i-1}$ and result of sliced $\sf ViT$ can be implemented by a uniform threshold circuit of depth $d_{\rm std}$ and size $nd < \poly(n)$.

    The addition between the result from previous step and $Z_{t_{i-1}}$ can be implemented by a uniform threshold circuit of depth $d_{\rm std}$ and size $nd < \poly(n)$.

    Since we have total $T$ iterations, we can have that the size of the uniform threshold circuit is bounded by $\poly(n)$, and the total depth of the circuit is $T(m(17d_{\rm std} + 8d_{\oplus} + 2d_{\rm sqrt} + d_{\rm exp}) + 5d_{\rm std}) = O(1)$ total depth since $T = O(1)$ and $m = O(1)$.
\end{proof}

Then, we extend this analysis to the sampling process for the second-order $\MeanFlow$ model.

\begin{theorem}[Second-order $\MeanFlow$ sampling with $T$-step Euler solver belongs to $\mathsf{TC}^0$, formal version of  Theorem~\ref{thm:2nd_thm:meanflow_sampling_tc0}]\label{thm:2nd_thm:meanflow_sampling_tc0:formal}
    Given $T \in \mathbb{Z}_+$ as the total iteration of sampling, and arbitrary $\{t_i\}_{i=1}^T$ that satisfy $t_{i} > t_{i+1}$ for all $i \in \{1,\ldots,T-1\}$ as timestep scheduler. 
    Assume the precision $p \leq \poly(n)$, the number of transformer layers $m = O(1)$, and $T = O(1)$.
    Then we can use a size bounded by $\poly(n)$ and constant depth $2T(m(17d_{\rm std} + 8d_{\oplus} + 2d_{\rm sqrt} + d_{\rm exp}) + 5d_{\rm std}) + Td_{\rm std}$ uniform threshold circuit to simulate the $\MeanFlow$ $T$-step sampling defined in Definition~\ref{def:2nd_meanflow_sampling_vit}.
\end{theorem}
\begin{proof}
    By using the result of Theorem~\ref{thm:meanflow_sampling_tc0}, we can apply a size bounded by $\poly(n)$ and a constant depth $T(m(17d_{\rm std} + 8d_{\oplus} + 2d_{\rm sqrt} + d_{\rm exp}) + 5d_{\rm std})$ uniform threshold circuit to simulate the computation of first two terms of Definition~\ref{def:2nd_meanflow_sampling_vit}.

    Then, for the last term, we can also use the result of Theorem~\ref{thm:meanflow_sampling_tc0}, except we need to add a scaler multiplication $(t_i - t_{i-1})/2$ before the scalar-matrix multiplication. We can simulate this operation by a uniform threshold circuit with size of $1 < \poly(n)$ and a constant depth of $Td_{\rm std}$.

    Thus, we can have that the size of the uniform threshold circuit is bounded by $\poly(n)$, and the total depth of the circuit is $2T(m(17d_{\rm std} + 8d_{\oplus} + 2d_{\rm sqrt} + d_{\rm exp}) + 5d_{\rm std}) + Td_{\rm std} = O(1)$ total depth since $T = O(1)$ and $m = O(1)$.
\end{proof}
\section{Provably Efficient Criteria}\label{sec:append:proof_efficiency}

This section provides the missing proofs from Section~\ref{sec:provably_efficiency}. In Section~\ref{sec:append:proof_efficiency:runtime_original}, we show the sampling runtime for the original second-order $\MeanFlow$. In Section~\ref{sec:append:proof_efficiency:runtime_fast}, we present the sampling runtime for both the fast second-order $\MeanFlow$. In Section~\ref{sec:append:proof_efficiency:error_bound}, we provide an error bound between the fast second-order $\MeanFlow$ layer and the second-order $\MeanFlow$ layer.

\subsection{Inference runtime of Second-order \texorpdfstring{$\MeanFlow$}{MeanFlow}}\label{sec:append:proof_efficiency:runtime_original}
In this section, we present the inference runtime of second-order $\MeanFlow$. The following lemma establishes the time complexity for the inference process of second-order $\MeanFlow$.
\begin{lemma}[Sampling runtime of Second-order $\MeanFlow$, formal version of Lemma~\ref{lem:runtime_old_2nd_meanflow}]\label{lem:runtime_old_2nd_meanflow:formal}
Consider the original second-order $\MeanFlow$ inference pipeline. The input is a tensor ${\sf X} \in \R^{h \times w \times c}$, where the height $h$ and width $w$ are both equal to $n$, and the number of channels $c$ is on the order of $O(\log n)$. 
The interpolated state at time $t \in [0,1]$ is denoted by ${\sf F}^t$, with ${\sf F}^1$ representing the final state. 
The model architecture consists of attention $(\mathsf{Attn})$, MLP $(\mathsf{MLP}(\cdot, c, d))$, and Layer Normalization $(\mathsf{LN})$ layers.
    
Based on these conditions, the inference time complexity of second-order $\MeanFlow$ is bounded by $O(n^{4+o(1)})$.
    
\end{lemma}
\begin{proof}
    {\bf Part 1: Runtime of first-order $\MeanFlow$ layer.} Each first-order $\MeanFlow$ layer processes an input of size $n \times n \times c$. The computation is dominated by the attention mechanism, which has a runtime of $O(n^4 c)$ per layer. The total runtime for all first-order layers is the sum over these individual complexities:
    \begin{align*}
        \mathcal{T}_{\mathsf{MF}}
        =&~ O(n^{4+o(1)}). 
    \end{align*}
    where the first step is obtained by simple algebra and $c = O(\log n)$. Here $\mathsf{MF}$ denotes $\MeanFlow$.

    {\bf Part 2: Runtime of second-order $\MeanFlow$ layer.} Similarly, each second-order $\MeanFlow$ layer operates on an input of size $n \times n \times c$, with the attention mechanism being the computational bottleneck at $O(n^4 c)$ per layer. The runtime complexity for all second-order layers is:
    \begin{align*}
        \mathcal{T}_{\mathsf{SMF}} 
        =&~ O(n^{4+o(1)}),
    \end{align*}
    wherhe the first is obtained by simple algebra and $c = O(\log n)$. Here $\mathsf{SMF}$ denotes second-order $\MeanFlow$.
    
    The total runtime for the entire high-order $\MeanFlow$ architecture is the sum of the runtimes of the first-order and second-order components.
    \begin{align*}
        \mathcal{T}_{\mathrm{ori}} 
        = \mathcal{T}_{\mathsf{MF}} + \mathcal{T}_{\mathsf{SMF}} = O(n^{4+o(1)}).
    \end{align*}

Thus, we complete the proof.
\end{proof}

\subsection{Inference runtime of Fast Second-order \texorpdfstring{$\MeanFlow$}{MeanFlow}}\label{sec:append:proof_efficiency:runtime_fast}
In this section, we analyze the inference runtime of fast second-order $\MeanFlow$.
\begin{lemma}[Sampling runtime of fast Second-order $\MeanFlow$, formal version of Lemma~\ref{lem:runtime_2nd_meanflow_fast}]\label{lem:runtime_2nd_meanflow_fast:formal}
Consider the fast second-order $\MeanFlow$ inference pipeline. It takes an input tensor ${\sf Z}_1 \in \R^{h \times w \times c}$, where the height $h=n$, width $w=n$, and the number of channels $c = O(\log n)$. The model architecture consists of attention $(\mathsf{AAttC})$, MLP $(\mathsf{MLP}(\cdot, c, d))$, and Layer Normalization $(\mathsf{LN})$ layers.
The interpolated state at time $t \in [0,1]$ is denoted by ${\sf Z}_t$, with ${\sf Z}_0$ as the final state. 
Based on these conditions, the total inference runtime complexity is bounded by $O(n^{2+o(1)})$.
\end{lemma}
\begin{proof}

{\bf Part 1: Runtime of fast first-order $\MeanFlow$ layer.}
Each fast first-order $\MeanFlow$ layer processes an input of size $n \times n \times c$. The layer's runtime is determined by its two main components: the MLP and the attention mechanism. First, the MLP layer requires $O(n^2 c)$ time. Given that the number of channels $c = O(\log n)$, this complexity is $O(n^{2+o(1)})$. Second, leveraging the method from \cite{as23}, the computationally expensive attention mechanism is accelerated from $O(n^4 c)$ to $O(n^2 c)$, which is also $O(n^{2+o(1)})$.

Since the runtime of both key components is $O(n^{2+o(1)})$, the complexity for a single layer is dominated by this term. Assuming a constant or polylogarithmic number of layers, the total runtime for all fast first-order $\MeanFlow$ layers is:
$
    \mathcal{T}_{\mathsf{MF}_\mathrm{fast}} = O(n^{2+o(1)}). 
$
    
{\bf Part 2: Runtime of fast second-order $\MeanFlow$ layer.} The runtime analysis for the fast second-order $\MeanFlow$ layers is similar to that of the first-order layers. Each layer processes an input of size $n \times n \times c$. The key computational components—the MLP layer and the accelerated attention mechanism from \cite{as23}—both operate with a complexity of $O(n^{2+o(1)})$.

Therefore, the total runtime complexity for all fast second-order layers is also bounded by this term:
$
    \mathcal{T}_{\mathsf{SMF}_\mathrm{fast}} = O(n^{2+o(1)}).
$ Here $\mathsf{SMF}_{\mathrm{fast}}$ denotes fast second-order $\MeanFlow$.

Thus, combining Part 1 and Part 2, we obtain the total runtime complexity for the fast high-order $\MeanFlow$, which is
\begin{align*}
    \mathcal{T}_{\mathrm{fast}} 
    =\mathcal{T}_{\mathsf{MF}_\mathrm{fast}} + \mathcal{T}_{\mathsf{SMF}_\mathrm{fast}} 
    = O(n^{2+o(1)}). 
\end{align*}

Thus, we complete the proof.
\end{proof}

\subsection{Error Bound between Fast and Standard Second-Order \texorpdfstring{$\MeanFlow$}{MeanFlow}}\label{sec:append:proof_efficiency:error_bound}
In this section, we show the error bound between fast and standard second-order $\MeanFlow$. The following lemma quantifies the error introduced by the approximate attention mechanism (Definition~\ref{def:appro_atten}).
\begin{lemma}[Error bound between fast second-order $\MeanFlow$ layer and standard second-order $\MeanFlow$ layer, formal version of Lemma~\ref{lem:error_bound_fast_2nd_meanflow_layers}]\label{lem:error_bound_fast_2nd_meanflow_layers:formal}
For the error analysis of the second-order $\MeanFlow$ Layer, we establish the following conditions. 
Let ${\sf Z}_1 \in \R^{h \times w \times c}$ be the input and let ${\sf Z}_1' \in \R^{h \times w \times c}$ be its approximation, such that the approximation error is bounded by $\|{\sf Z}_1 - {\sf Z}_1'\|_{\infty} \leq \epsilon$ for some small constant $\epsilon > 0$. 

The analysis considers interpolated inputs ${\sf Z}_t, {\sf Z}_{\mathrm{fast},t} \in \R^{h \times w \times c}$ over a time step $t \in [0, 1]$. These inputs are processed by a standard second-order $\MeanFlow$ layer, ${\sf SMF}(\cdot, \cdot, \cdot)$, and a fast variant, ${\sf SMF}_{\mathrm{fast}}(\cdot, \cdot, \cdot)$,
where the fast variant substitute $\sf Attn$ operation with $\sf AAttC$ (Definition~\ref{def:appro_atten}).
The approximation is achieved via a polynomial $f$ of degree $g$ and low-rank matrices $U, V \in \R^{hw \times k}$. 

We assume all matrix entries are bounded by a constant $R > 1$. Crucially, we also make assumption that the LayerNorm function ${\sf LN}(\cdot)$ (Definition~\ref{def:layer_norm}), does not exacerbate error propagation; that is, if $\|{\sf Z}_1' - {\sf Z}_1\|_{\infty} \le \epsilon$, then it follows that $\|{\sf LN}({\sf Z}_1') - {\sf LN}({\sf Z}_1)\|_{\infty} \le \epsilon$.

Then, we have
\begin{align*}
    \| {\sf SMF}_{\mathrm{fast}}({\sf Z}_{\mathrm{fast},t}, t, r) - {\sf SMF}({\sf Z}_t, t, r) \|_{\infty} \leq O(c^2kR^{g+2}) \cdot \epsilon.
\end{align*}
\end{lemma}

\begin{proof}
We begin by establishing a bound on the difference between the interpolated inputs, ${\sf Z}_{\mathrm{fast},t}$ and ${\sf Z}_t$. Given that $t \in [0, 1]$ and the input approximation error is bounded by $\|{\sf Z}_1' - {\sf Z}_1\|_{\infty} \leq \epsilon$, we have:
\begin{align}\label{eq:z_fast_t_zt_eps}
    \|{\sf Z}_{\mathrm{fast},t} - {\sf Z}_t\|_{\infty} 
    = & ~ \|t({\sf Z}_1' - {\sf Z}_1)\|_{\infty} \notag \\
    \leq & ~ \|{\sf Z}_1' - {\sf Z}_1\|_{\infty} \notag \\
    \leq & ~ \epsilon,
\end{align}
where the first step follows from the definition of linear interpolated flows, the second step follows from $t \in [0,1]$, and the second step follows from $\|{\sf Z}_0' - {\sf Z}_0\|_{\infty} \le \epsilon$.

Then, according to Definition~\ref{def:2nd_meanflow_sampling_vit}, the predicted flows are computed using:
\begin{align} \label{eq:zr_smf}
{\sf Z}_r 
= & ~ {\sf SMF}({\sf Z}_t, t, r), \\
{\sf Z}_{\mathrm{fast},r}
= & ~ {\sf SMF}_\mathrm{fast}({\sf Z}_{\mathrm{fast},t}, t, r). \label{eq:z_fastr_smf}
\end{align}

Let $||$ denote the concatenation operation. Next, we bound the error for two different outputs:
\begin{align*}
     \|{\sf Z}_r -  {\sf Z}_{\mathrm{fast},r}\|_\infty
    & ~ =  \|{\sf SMF}({\sf Z}_t, t, r) - {\sf SMF}_\mathrm{fast}({\sf Z}_{\mathrm{fast},t}, t, r)\|_\infty \\
    & ~ =  \|{\sf Z}_t + (r - t) {\sf ViT}_1({\sf Z}_t ~ || ~ r{\bf 1}_n ~ || ~ t{\bf 1}_n)_{*, 1:d} 
     + \frac{1}{2}(r - t)^2 {\sf ViT}_2({\sf Z}_t ~ || ~ r{\bf 1}_n ~ || ~ t{\bf 1}_n)_{*, 1:d} 
     - {\sf Z}_{\mathrm{fast},t}\\ 
     & ~ - (r - t) {\sf ViT}_{\mathrm{fast}, 1}({\sf Z}_{\mathrm{fast},t} ~ || ~ r{\bf 1}_n ~ || ~ t{\bf 1}_n)_{*, 1:d}
     - \frac{1}{2}(r - t)^2 {\sf ViT}_{\mathrm{fast},2}({\sf Z}_{\mathrm{fast},t} ~ || ~ r{\bf 1}_n ~ || ~ t{\bf 1}_n)_{*, 1:d}\|_\infty \\
    & ~ \leq \|{\sf Z}_t - {\sf Z}_{\mathrm{fast},t}\|_\infty
    + (r - t)\|({\sf ViT}_1({\sf Z}_t ~ || ~ r{\bf 1}_n ~ || ~ t{\bf 1}_n) 
     - {\sf ViT}_{\mathrm{fast}, 1}({\sf Z}_{\mathrm{fast},t} ~ || ~ r{\bf 1}_n ~ || ~ t{\bf 1}_n))_{*, 1:d}\|_\infty \\
    & ~ + \frac{1}{2}(r - t)^2 \|({\sf ViT}_2({\sf Z}_t ~ || ~ r{\bf 1}_n ~ || ~ t{\bf 1}_n)
    - {\sf ViT}_{\mathrm{fast},2}({\sf Z}_{\mathrm{fast},t} ~ || ~ r{\bf 1}_n ~ || ~ t{\bf 1}_n))_{*, 1:d} \|_\infty\\
    & ~ \leq \epsilon + A + B,
\end{align*}
where the first step follows from substituting ${\sf Z}_{\mathrm{r}}$ and ${\sf Z}_{\mathrm{fast,r}}$ with Eq.~\eqref{eq:zr_smf} and Eq.~\eqref{eq:z_fastr_smf}, the second step follows from Definition~\ref{def:2nd_meanflow_sampling_vit}, the third step follows from triangle inequality, the last step follows from Eq.~\eqref{eq:z_fast_t_zt_eps} and defining two shorthand notations:
\begin{align*}
    A := & ~ (r - t) \|({\sf ViT}_1({\sf Z}_t ~ || ~ r{\bf 1}_n ~ || ~ t{\bf 1}_n)  
     - {\sf ViT}_{\mathrm{fast}, 1}({\sf Z}_{\mathrm{fast},t} ~ || ~ r{\bf 1}_n ~ || ~ t{\bf 1}_n))_{*, 1:d}\|_\infty,\\
    B := & ~\frac{1}{2}(r - t)^2 \|({\sf ViT}_2({\sf Z}_t ~ || ~ r{\bf 1}_n ~ || ~ t{\bf 1}_n)
     - {\sf ViT}_{\mathrm{fast},2}({\sf Z}_{\mathrm{fast},t} ~ || ~ r{\bf 1}_n ~ || ~ t{\bf 1}_n))_{*, 1:d}\|_\infty.
\end{align*}
Consider the $Y$ in the first layer of $\sf ViT$,
\begin{align}\label{eq:fast_y_bound}
 \|Y_1 - Y_{\mathrm{fast},1}\|_\infty 
 & ~ = \| {\sf Attn}({\sf LN}_{1,1}({\sf Z}_t ~ || ~ r{\bf 1}_n ~ || ~ t{\bf 1}_n)) + {\sf Z}_t
 - {\sf AAttC}({\sf LN}_{1,1}({\sf Z}_{\mathrm{fast},t} ~ || ~ r{\bf 1}_n ~ || ~ t{\bf 1}_n)) - {\sf Z}_{\mathrm{fast},t}\|_\infty \notag  \\
& ~ \leq \| {\sf Attn}({\sf LN}_{1,1}({\sf Z}_t ~ || ~ r{\bf 1}_n ~ || ~ t{\bf 1}_n)) 
 - {\sf AAttC}({\sf LN}_{1,1}({\sf Z}_{\mathrm{fast},t} ~ || ~ r{\bf 1}_n ~ || ~ t{\bf 1}_n)) \|_\infty
 + \|{\sf Z}_t - {\sf Z}_{\mathrm{fast},t}\|_\infty \notag \\
& ~ \leq O(kR^{g+1}c)\cdot\epsilon + \epsilon,
\end{align}
where the first step is due to Definition~\ref{def:vit_with_residual}, the second step is obtained by triangular inequality, and the last step follows from Lemma~\ref{lem:error_analysis_of_attention} and assumption in the lemma.

Then, for $X$ in the first layer of $\sf ViT$,
\begin{align*}
     \|X_1 - X_{\mathrm{fast},1}\|_\infty 
     & ~ = \|\mathsf{MLP}_1(\mathsf{LN}_{1,2}(Y_1)) + Y_1 
     - \mathsf{MLP}_1(\mathsf{LN}_{1,2}(Y_{\mathrm{fast},1})) - Y_{\mathrm{fast},1}\|_\infty \\
    & ~ \leq \|\mathsf{MLP}_1(\mathsf{LN}_{1,2}(Y_1)) - \mathsf{MLP}_1(\mathsf{LN}_{1,2}(Y_{\mathrm{fast},1})) \|_\infty 
     + \| Y_1 -  Y_{\mathrm{fast},1}\|_\infty \\
    & ~ \leq cR(O(kR^{g+1}c)\cdot\epsilon + \epsilon) + O(kR^{g+1}c)\cdot\epsilon + \epsilon \\
    & ~ \leq O(c^2kR^{g+2}) \cdot \epsilon,
\end{align*}
where the first step follows from Definition~\ref{def:vit_with_residual}, the second step is due to triangular inequality, the third step is obtained by Lemma~\ref{lem:error_analysis_mlp}, Eq.~\eqref{eq:fast_y_bound}, and assumption in the lemma, and the last step follows from simple algebra.

Thus we have
\begin{align*}
    A \leq & ~ (r-t)O(c^2kR^{g+2}) \cdot \epsilon = O(c^2kR^{g+2}) \cdot \epsilon,\\
    B \leq & ~ \frac{1}{2}(r-t)^2 O(c^2kR^{g+2}) \cdot \epsilon = O(c^2kR^{g+2}) \cdot \epsilon.
\end{align*}

Thus we can combine the final result,
\begin{align*}
    \|{\sf Z}_r -  {\sf Z}_{\mathrm{fast},r}\|_\infty \leq \epsilon + A + B \leq O(c^2kR^{g+2}) \cdot \epsilon.
\end{align*}

Thus we complete the proof.
\end{proof}

\newpage
\ifdefined\isarxiv
\bibliographystyle{alpha}
\bibliography{ref}
\else
{\small
\bibliographystyle{ieeenat_fullname}
\bibliography{ref}
}
\fi

\end{document}